\documentclass[letterpaper]{article} 
\usepackage{aaai25}  
\usepackage{times}  
\usepackage{helvet}  
\usepackage{courier}  
\usepackage[hyphens]{url}  
\usepackage{graphicx} 
\usepackage{natbib}  
\usepackage{caption} 
\usepackage{soul}
\usepackage[utf8]{inputenc}
\usepackage{amsmath}
\usepackage{amsthm}
\usepackage{amssymb}

\usepackage{booktabs}
\usepackage[american]{babel}
\usepackage{microtype}
\usepackage[switch]{lineno}
\usepackage{multirow}
\usepackage{color}
\usepackage{bbding}
\usepackage[x11names]{xcolor}

\usepackage[caption=false,font=footnotesize,labelfont=rm,textfont=rm]{subfig}
\usepackage[american]{babel}
\usepackage{microtype}
\usepackage{pifont}

\usepackage{algorithm}
\usepackage{algorithmic}

\newtheorem{theorem}{Theorem}

\usepackage{bm}
\newcommand{\x}{{\bm x}}
\newcommand{\y}{{\bm y}}
\newcommand{\z}{{\bm z}}

\newcommand{\C}{{\bm c}}

\newcommand{\T}{{\top}}

\newcommand{\ie}{\textrm{i.e.}}

\newcommand{\eg}{\textrm{e.g.}}
\newcommand{\wrt}{\textrm{w.r.t.\,}}

\usepackage{newfloat}
\usepackage{listings}
\DeclareCaptionStyle{ruled}{labelfont=normalfont,labelsep=colon,strut=off} 
\floatstyle{ruled}
\newfloat{listing}{tb}{lst}{}
\floatname{listing}{Listing}

\title{Hierarchical Consensus Network for Multiview Feature Learning}
\author{
Chengwei Xia\textsuperscript{\rm 1}, 
Chaoxi Niu\textsuperscript{\rm 2}, 
Kun Zhan\textsuperscript{\rm 1}\thanks{Corresponding Author.}
}
\affiliations{
\textsuperscript{\rm 1}School of Information Science and Engineering, Lanzhou University\\
\textsuperscript{\rm 2} Australian Artificial Intelligence Institute, University of Technology Sydney \\
https://github.com/kunzhan/HCN
}

\begin{document}
\maketitle
\begin{abstract}
Multiview feature learning aims to learn discriminative features by integrating the distinct information in each view. However, most existing methods still face significant challenges in learning view-consistency features, which are crucial for effective multiview learning. Motivated by the theories of CCA and contrastive learning in multiview feature learning, we propose the hierarchical consensus network (HCN) in this paper. The HCN derives three consensus indices for capturing the hierarchical consensus across views, which are classifying consensus, coding consensus, and global consensus, respectively. Specifically, classifying consensus reinforces class-level correspondence between views from a CCA perspective, while coding consensus closely resembles contrastive learning and reflects contrastive comparison of individual instances. Global consensus aims to extract consensus information from two perspectives simultaneously. By enforcing the hierarchical consensus, the information within each view is better integrated to obtain more comprehensive and discriminative features. The extensive experimental results obtained on four multiview datasets demonstrate that the proposed method significantly outperforms several state-of-the-art methods.
\end{abstract}
\begin{figure}[!t]
\centering
\includegraphics[width=0.78\linewidth]{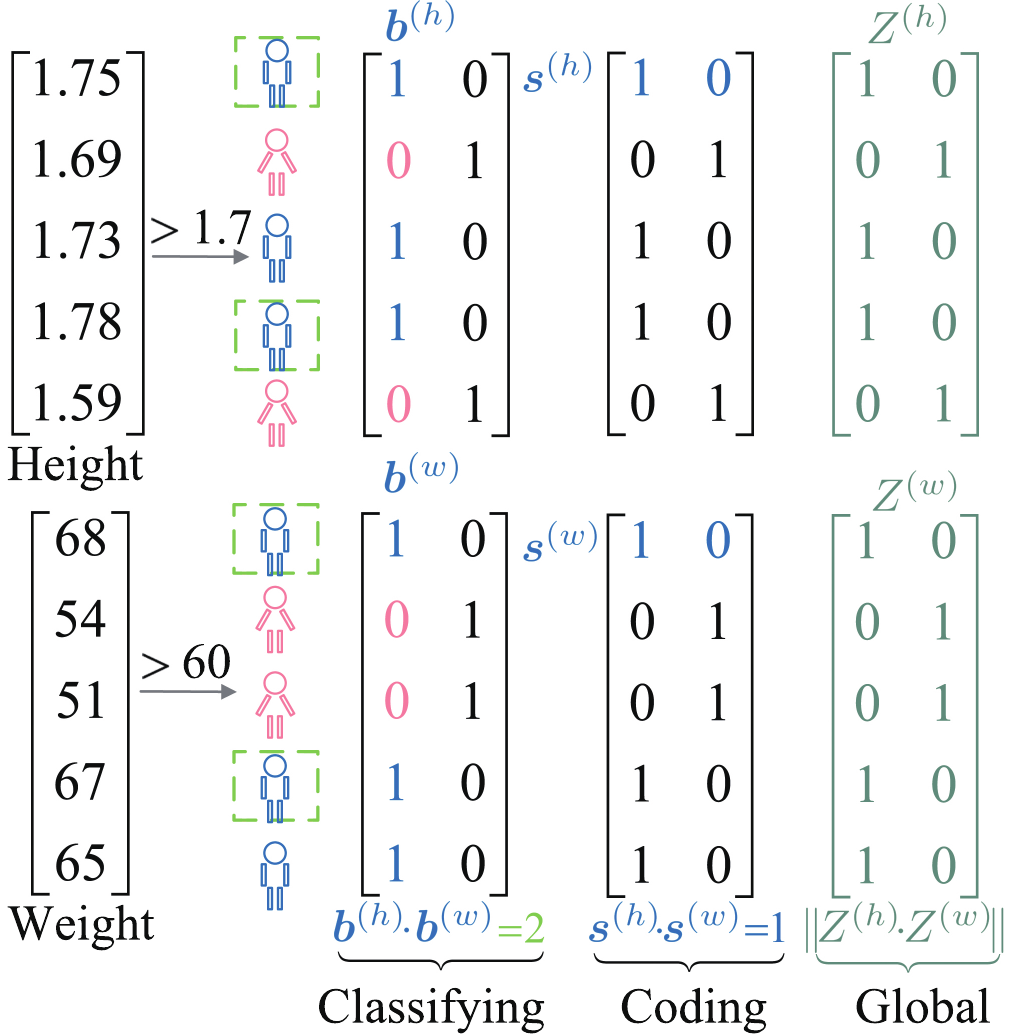}%
\caption{The hierarchical consensus: Consider a scenario with five students, where we collect their height and weight. We employ two simple classifiers: one based on height (assigning `boy' if height $>$1.7m and `girl'; otherwise) and another based on weight (assigning `boy' if weight $>$60kg and `girl' otherwise). Each view produces binary predictions. Our hierarchical consensus objective involves deriving consensus indices through inner product computations. (1) \textbf{Classifying Consensus}: In the first column of the matrix, $\bm{b}^{(h)}$ and $\bm{b}^{(w)}$ denote boys' predictions from the two views. The result of 2 counts the number of boys. The objective of classifying consensus is to align predictions for gender quantity. (2) \textbf{Coding Consensus}: Moving to the first row of the second matrix, $\bm{s}^{(h)}$ and $\bm{s}^{(w)}$ represent the gender coding for a same student. The result of 1 indicates consensus prediction of the student. The objective of coding consensus is to align gender coding of the student. (3) \textbf{Global Consensus}: To the whole matrix perspective, a global consensus is established.}
\label{moti}
\end{figure}

\section{Introduction}
Multiview data are universal in the real world and typically contain multiple views of the same underlying semantic information. Generally, a single view cannot provide adequate information for feature learning. To address this, multiview feature learning aims to integrate the common and complementary information contained in multiple views to generate more discriminative data features than a single view~\cite{chen2023deep,Xu_2022_CVPR}.   

Multiview feature learning methods are roughly divided into two categories, \ie, traditional and deep methods. However, traditional methods typically suffer from limited representation capacity, and high computational complexity for complex data scenarios. To alleviate these problems, many works propose to perform multiview feature learning based on deep neural networks and achieve superior performance compared to the traditional methods. 

Most deep learning methods focus on learning consistency between different view features to obtain a common representation through concatenation or adaptive weighted fusion~\cite{trostenMVC,yang2023dealmvc,xu2024investigating}. Canonical Correlation Analysis (CCA)~\cite{hotelling1936relations} and co-training~\cite{blum1998combining} are two representative methods in consistency learning, achieving promising results in exploring common features across views~\cite{andrew2013deep,wang2015deep,Lin2022}.
As a promising unsupervised representation learning method, contrastive learning has also gained increasing attention \cite{chen2020simple,he2020momentum}, and several multiview feature learning methods based on contrastive learning have been proposed, \eg, \cite{trostenMVC} directly conducts alignment-based instance-level contrast across multiple views. ~\cite{chen2023deep,Xu_2022_CVPR} rely on additional components such as an MLP projector to obtain a cluster assignment probability, and achieve consistency by contrasting the cluster assignments across views. Moreover, ~\cite{yan2023gcfagg,xu2024self} propose a selection of negative pairs and reweighting strategies to improve contrastive learning performance under multiview scenarios, resulting in additional costs.

In this paper, building on the insight of CCA and contrastive learning in multiview feature learning, we propose the hierarchical consensus network (HCN) which derives three consensus indices to explore hierarchical consensus between views, \ie, classifying, coding, and global consensus. In Figure~\ref{moti}, our motivation is clearly and intuitively presented, making it easy to see the connections among HCN, CCA~\cite{hotelling1936relations}, and contrastive learning. The objective of classifying consensus is conceptually similar to CCA, while the objective of coding consensus closely resembles contrastive learning. The objective of global consensus aims to capture consensus information from two perspectives simultaneously. Specifically, classifying consensus aims to reinforce class-level correspondence between views. Coding consensus is proposed to achieve contrastive comparison of individual instances between different views. Global consensus minimizes the difference between the different views. Overall, HCN fully characterizes the consensus between views from different perspectives.

To apply the proposed HCN for multiview learning, we employ a view-specific autoencoder to learn the distinct and common information within each view. Besides, we further apply the augmentation technique to increase the training samples for each view. Original and augmented data are fed into the view-specific encoder to obtain latent features. Then, hierarchical consensus learning is conducted on the latent features of the original and the augmented data across views. Specifically, for classifying consensus learning, we minimize the conditional entropy of the class probability in one view conditioned by the other view. For the coding consensus, we employ a weak-to-strong
pseudo-supervision cross-entropy between the original and the augmented data in each view. For the global consensus, we minimize the difference between the two latent features. Compared to state-of-the-art (SOTA) methods, HCN explores consensus information in multiview feature learning from novel insights, with a lower complexity, and does not require the introduction of additional components or sample selection and weighting strategies.

The main contributions are summarized as follows: \ding{182} We introduce hierarchical consensus to explore multiple consistencies between views, providing promising insights into multiview feature learning. \ding{183} We design the Hierarchical Consensus Network (HCN) for multiview feature learning. HCN effectively learns a comprehensive and discriminative feature by capturing hierarchical consensus among multiple views. \ding{184} Experiments on four multiview datasets demonstrate the effectiveness of HCN over other SOTA baselines. 
\begin{figure*}[!t]
\centering
\includegraphics[width=0.96\textwidth]{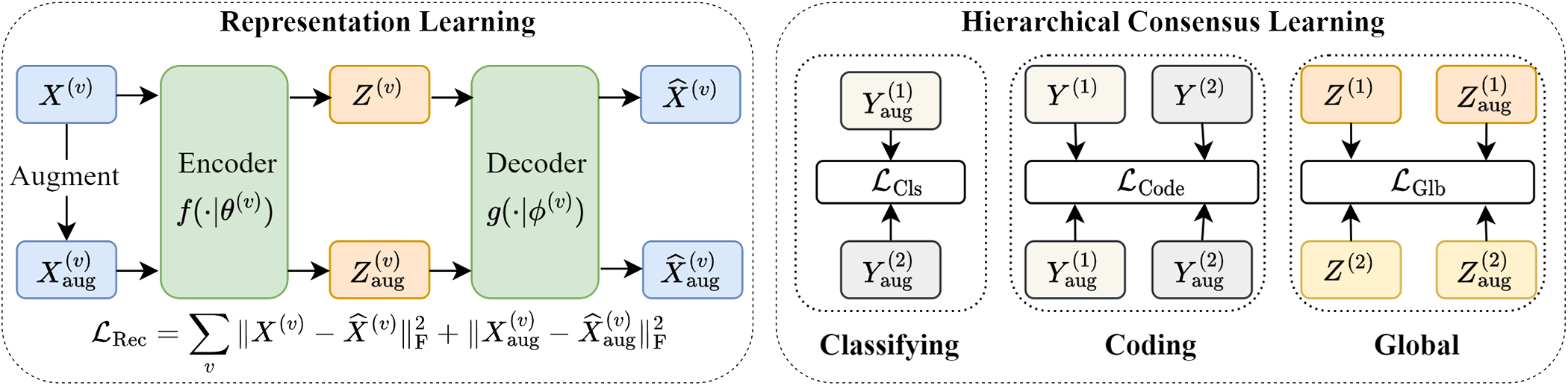}
\caption{The HCN framework. Each view contains a view-specific autoencoder, \ie, an encoder $f(\cdot|\theta^{(v)})$ and decoder $g(\cdot|\phi^{(v)})$. The representations $Z^{(v)}$ and $Z^{(v)}_{\text{aug}}$ are learned by minimizing the construction error $\mathcal{L}_{\text{Rec}}$. Besides, $Z^{(v)}$ and $Z^{(v)}_{\text{aug}}$ are fed into the softmax to obtain the class posterior probabilities, \ie, $Y^{(v)}$ and $Y^{(v)}_{\text{aug}}$. Given the representations and class probabilities of the two views, HCN aims to capture hierarchical consensus between them. Specifically, classifying consensus ensures the consistency of class distributions across views, coding consensus enhances the same prediction for the same sample, and global consensus minimizes the difference between learned representations from different views.}
\label{framework2}
\end{figure*}

\section{Related Work}
Most multiview feature learning methods suffer from drawbacks such as high complexity and limited performance~\cite{chen2021multiview,zhan2018multiview,ijcai2019p510,liu2020efficient}. Recently, several consistency-based multiview feature learning methods have been proposed, aiming to maximize consistency between different views. Inspired by the strategy to maximize view consistency between two sets in CCA~\cite{hotelling1936relations}, \cite{andrew2013deep} maps multiview features into a common space and concatenates the low-dimensional features as the common representation. ~\cite{wang2015deep} further introduce autoencoders in multiview feature learning compared with~\cite{andrew2013deep}.
Besides, by leveraging co-training strategy~\cite{blum1998combining}, \cite{Lin2022} design dual contrastive prediction to learn the cross-view consistency.

As one of the most effective consistent learning paradigms, contrastive learning has achieved SOTA performance~\cite{chen2020simple,he2020momentum}. The basic idea of contrastive learning is learning a feature space from raw data by maximizing the similarity between positive pairs while minimizing that of negative pairs. In recent, some methods have shown the success of contrastive learning in multiview feature learning~\cite{trostenMVC,Xu_2022_CVPR}, where similarities of positive pairs are maximized and that of negative pairs are minimized via NT-Xent~\cite{chen2020simple}. \cite{trostenMVC} learns common representation by aligning representations from different views at the sample level. ~\cite{yan2023gcfagg} learns the global structural relationships between samples and utilizes them to obtain consistent data representations. Simultaneously, structural information is utilized to select negative pairs for cross-view contrastive learning. In~\cite{xu2024self}, the weights are optimized based on the discrepancy between pairwise representations, performing self-weighted contrastive learning. Considering consistency between the cluster assignments among multiple views, \cite{chen2023deep} proposes a cross-view contrastive learning method to learn view-invariant representations by contrasting cluster assignments among multiple views. Moreover, contrastive clustering~\cite{li2021contrastive} designed for single-view clustering tasks, constructs two distinct views through data augmentation and subsequently projects them into feature space. Using two separate projection heads, the method conducts contrastive learning at different levels in the row and column space to jointly learn representations and cluster assignments. \cite{Xu_2022_CVPR} conducts multi-level features contrast from multiple views to achieve consistency. Furthermore, some recent contrastive learning works, notably BYOL~\cite{grill2020bootstrap} and SimSiam~\cite{chen2021exploring}, have shown the remarkable ability to learn powerful representations using only positive pairs, which has proven to be a simple and effective method~\cite{Tian2021}.

\section{Hierarchical Consensus Learning}
We incorporate a hierarchical structure by defining objectives at the matrix, row, and column levels: maximizing the global consensus between the matrices of dual-view learned features aligns the overall representations; maximizing the pairwise coding consensus between rows enables the contrastive comparison of individual instances; and maximizing the classifying consensus between columns aligns class-level representations, similar CCA.
\subsection{Notations}
A multiview dataset $\mathcal{X} =\{X^{(1)},X^{(2)},\ldots,X^{(n_v)}\}$ typically consists of $n$ samples across $n_v$ views. Specifically, the data of $v$-th view is represented as $X^{(v)}=[\x_1^{(v)},\x_2^{(v)},\ldots,\x_n^{(v)}]^\T\in\mathbb{R}^{n\times d_v}$, where $\x^{(v)}$ denotes the samples with $d_v$ dimensional raw feature. Since each view contains distinct information about the dataset, we use $f(\cdot|\theta^{(v)})$ to denote nonlinear mappings implemented by the encoder for view $v$, with the corresponding parameters $\theta^{(v)}$. Let $Z^{(v)} = [\z_1^{(v)},\z_2^{(v)},\ldots,\z_n^{(v)}]^\T \in \mathbb{R}^{n\times k}$ denote the learned feature from the input in view $v$ by $f(X^{(v)}|\theta^{(v)})$ where $k$ denotes the feature dimension. We aim to capture hierarchical consensus between the learned features from multiple views in an unsupervised way and generate a more comprehensive and discriminative feature for $\mathcal{X}$.
\subsection{Classifying Consensus Learning}
We begin with classifying consensus learning between $Z^{(1)}$ and $Z^{(2)}$ which aims to enhance the consistency of class distribution between views. Taking the first view as an example, we model the class posterior probabilities $y_{ij}^{(1)}$ or $p(c_j^{(1)}|\x_i^{(1)})$ of sample $\x_{i}^{(1)}$ belonging to the $j$-th class by applying the softmax layer on $Z^{(1)}$: $y_{ij}^{(1)}=p(c_j^{(1)}|\x_i^{(1)})=\frac{\exp(z_{ij}^{(1)})}{\sum_j \exp(z_{ij}^{(1)})}$ where $z_{ij}^{(1)}$ represents the $j$-th entry of the $i$-th sample in $Z^{(1)}$. In the same way, we obtain $y_{ij}^{(2)}$ or $p(c_j^{(2)}|\x_i^{(2)})$ based on $Z^{(2)}$ for the second view.

Referring to the similarity kernel function defined in~\cite{Haussler1999} and~\cite{BishopPRML}, the joint probability of sample belonging to distinct classes is given by
\begin{align}
\widetilde p(c_{j},c_{h})=\int
p(c_{j}|\x_i)
p(c_{h}|\x_i)p(\x_i){\rm d}\x_i
\label{kernel}
\end{align}
where $c_j$ and $c_h$ denote different classes. Similar to the classifying consensus shown in Figure~\ref{moti}, we define classifying consensus probability. By extending Eq.~\eqref{kernel} to the discrete scenario and multiview learning, the joint class probability of discrete random variables between two different views~\cite{ji2019invariant,lin2021completer} is defined by
\begin{align}
\widetilde p(c_{j}^{(1)},c_{h}^{(2)})
=\frac{1}{n}\sum_{i=1}^n y_{ij}^{(1)}y_{ih}^{(2)}\,,\label{column}
\end{align}
Eq.~\eqref{column} performs a column-wise inner product between cross-view matrices $Y^{(1)}=[y_{ij}^{(1)}]$ and $Y^{(2)}=[y_{ij}^{(2)}]$ to measure similarity, similar to CCA. This measure is normalized by
\begin{align}
p(c_{j}^{(1)},c_{h}^{(2)})
=\frac{\widetilde p(c_{j}^{(1)},c_{h}^{(2)})}{\sum_{j,h=1}^{k}\widetilde p(c_{j}^{(1)},c_{h}^{(2)})}\,.
\label{joint_p}
\end{align}
Based on the obtained cross-view joint probability $p(c_{j}^{(1)},c_{h}^{(2)})$, it is straightforward to derive $p(\C^{(1)},\C^{(2)})$, $p(\C^{(1)})$, $p(\C^{(2)})$, $p(\C^{(2)}|\C^{(1)})$, and $p(\C^{(1)}|\C^{(2)})$. 

For the classifying consensus, the softmax operation ensures independence between the prediction columns, satisfying the decorrelation constraint in CCA. Similar to the CCA objective of maximizing cross-view correlations, classifying consensus aims to align the predictions \wrt the same class between two views. To achieve this goal, we minimize the conditional entropy of the class probability of one view conditioned by that of the other view. In other words, if the classes of data points in one view are known, the uncertainty about the classes of corresponding data points in the other view is minimized by the conditional entropy. Then, the classifying consensus is formulated as the conditional entropy between two views:   
\begin{align}
\min {\rm H} (\C^{(1)}|\C^{(2)})+{\rm H} (\C^{(2)}|\C^{(1)})\,
\label{infor1}
\end{align}
where ${\rm H}(\cdot)$ denotes the entropy operator. However, directly optimizing Eq.~\eqref{infor1} induces the risk that all samples are assigned to a particular class. To address this issue, we introduce a regularization term to encourage the class distribution to be balanced in each view. Specifically, we propose to maximize the entropy of the class probability in each view as 
\begin{align}
\max {\rm H} (\C^{(1)})+{\rm H} (\C^{(2)})\,.
\label{infor2}
\end{align}
By combining Eqs.~\eqref{infor1} and \eqref{infor2}, we obtain the final objective for classifying consensus learning:
\begin{align}
\mathcal L_{\rm cls} 
= \alpha{\rm H}(\C^{(1)}|\C^{(2)}) 
- \beta {\rm H}(\C^{(1)})
- \gamma{\rm H}(\C^{(2)})
\end{align}
where $\alpha$, $\beta$ and $\gamma$ are trade-off hyperparameters. In this way, we achieve the classifying consensus between views and avoid assigning all samples to a particular class.
\subsection{Coding Consensus Learning}
Different views characterize the same sample from different perspectives with the semantics of the sample shared between views. Based on this, we further propose coding consensus learning. Formally, given the class posterior probability obtained from softmax layer, we employ the cross-entropy for coding consensus learning~\cite{Hanijcai2024}:
\begin{align}
\mathcal L_{\rm code}=-\sum_{i=1}^n\bigl((\y^{(2)}_{i})^\T\log \y^{(1)}_{i}+(\y^{(1)}_{i})^\T\log \y^{(2)}_{i}\bigr)
\label{ic}
\end{align}
where $\y$ is the softmax prediction $\x$.
\subsection{Global Consensus Learning}
In addition to classifying and coding consensus learning based on the posterior probabilities of data classes, global consensus learning aims to minimize the difference between the two latent features encoded from different views. Thus, the objective of global consensus is given by:
\begin{align}
\mathcal L_{\rm glb} = -{\rm tr}\Bigl(\bigl({Z^{(1)}}\bigr)^\T {Z}^{(2)} \Bigr)\, .\label{matrix}
\end{align}
where ${\rm tr}(\cdot)$ denotes the trace operation.
\subsection{Hierarchical Structure}
Eq.~\eqref{infor1} defines the column-wise classifying consensus learning objective, corresponding to the column vectors shown in Figure~\ref{moti}. Eq.~\eqref{ic} represents a row-wise contrastive comparison of individual instances, as indicated by the row vectors in Figure~\ref{moti}. Lastly, Eq.~\eqref{matrix} maximizes the global consensus, aligning overall representations between views, illustrated by the entire matrix in Figure~\ref{moti}.

The classifying consensus aligns column vectors from different views, similar to CCA, which maximizes cross-view correlations. This correlation maximization is achieved in Eq.~\eqref{column}. Additionally, the decorrelation constraint in each view, typical of CCA, is incorporated here through the softmax operation, which enforces independence across columns in each view. Thus, classifying consensus closely aligns with the principles of CCA.

By optimizing the coding consensus objective, we ensure that the coding assignment of sample $i$ remains consistent across views, in line with the definition of multiview data. The connection between coding consensus learning and contrastive learning is formally established in Theorem~\ref{the1}.
\begin{theorem}\label{the1}
Coding-consensus learning is equivalent to contrastive learning with positive pairs.
\end{theorem}
The proofs and further corollaries are in Appendix~B.2.

As shown in Figure~\ref{moti}, if we denote the learned matrix as $Z=[\bm b_1,\ldots,\bm b_k]=[{\bm s}_1,\ldots,\bm s_n]^\T$, we have the relationship:
\begin{align}
{\rm tr}\bigl((Z^{(1)})^\T Z^{(2)}\bigr)
=\sum_{j=1}^k ({\bm b}_j^{(1)})^\T{\bm b}_j^{(2)} 
=\sum_{i=1}^n ({\bm s}^{(1)}_i)^\T{\bm s}^{(2)}_i
\label{connetions}
\end{align}
where we focus on the inner products, and then argue that the first term of Eq.~\eqref{connetions} operates at the whole matrix level, capturing global alignment; the second term aligns cross-view columns, reinforcing class-level correspondence; and the third term reflects the effect of instance-wise row contrastive comparison, focusing on individual sample alignment. Viewed from another perspective, the global consensus simultaneously captures both the coding and classifying effects, as shown in Eq.~\eqref{connetions}.
\section{Hierarchical Consensus Network}
The general framework is shown in Figure~\ref{framework2}. Given multiview data, a view-specific autoencoder is employed to exploit the distinct information~\cite{lin2021completer} within each view. Specifically, we denote the encoder as $f(X^{(v)}|\theta^{(v)})$ that takes the data $X^{(v)}$ as input and output the latent feature $Z^{(v)}$ of view $v$. Inversely, the decoder $g(Z^{(v)}|\phi^{(v)})$ aims to reconstruct the input data based on the latent feature, denoted as $\widehat{X}^{(v)}$. We implement the encoder with a four-layer MLP followed by a softmax. The decoders have the same architecture as the encoders.
Then, a reconstruction objective is employed to learn the semantic information of the input.

For each view, drop-feature augmentation is applied to the input data $X^{(v)}$ to enrich to input data of view $v$, which is implemented as randomly dropping certain feature dimensions. Formally, we first sample a random vector ${\bm m} \in \{0, 1\}^{d_v}$ with each entry being drawn from a Bernoulli distribution independently, \ie, $m_j \sim {\rm Bern}(1-\rho)$, where $m_j$ is the $j$-th entry of ${\bm m}$ and $\rho$ is the drop rate. Therefore, the augmented data of view $v$ is represented as $X_{\rm aug}^{(v)} = X^{(v)}[:,\mathbb{I}({\bm m})]$ where $\mathbb{I}(\cdot)$ indicates whether each entry of ${\bm m}$ is 1 or 0. The latent feature $Z_{\rm aug}^{(v)}$ of $X_{\rm aug}^{(v)}$ is also obtained from the view-specific encoder, \ie, $Z_{\rm aug}^{(v)} = f(X_{\rm aug}^{(v)}|\theta^{(v)})$ and is used to reconstruct the augmented data. 

Specifically, the reconstruction objective is defined by:
\begin{align}
\mathcal{L}_{\rm Rec} =\sum\limits_{v=1}^{n_v} \bigl(\|X^{(v)}-\widehat{X}^{(v)}\|_{\rm F}^2+\|X^{(v)}_{\rm aug}-\widehat{X}^{(v)}_{\rm aug}\|_{\rm F}^2 \bigr)\,.
\label{recloss}
\end{align}
Given the learned latent features $Z^{(1)}$, $Z^{(2)}$, $Z_{\rm aug}^{(1)}$, and $Z_{\rm aug}^{(2)}$, we obtain the corresponding class probabilities by applying the softmax operation on them, \ie, $Y^{(1)}$, $Y^{(2)}$, $Y_{\rm aug}^{(1)}$ and $Y_{\rm aug}^{(2)}$. Then, we detail how to conduct consensus learning between and within the two views. 

First, we perform the classifying consensus learning between the augmented data from two views. Formally, the joint class probability $p(\C^{(1)}_{\rm aug}|\C^{(2)}_{\rm aug})$ is obtained by Eq.~\eqref{joint_p}, and the marginal probability distributions $p(\C^{(1)}_{\rm aug})$ and $p(\C^{(2)}_{\rm aug})$ are calculated by summing along rows and columns of the joint probability matrix respectively. Then, the classifying consensus objective is given by:
\begin{equation}
 \mathcal L_{\rm Cls} = \sum_{\substack{u,v=1\\u> v}}^{n_v}\alpha{\rm H}(\C^{(u)}_{\rm aug}|\C^{(v)}_{\rm aug})- \beta {\rm H}(\C^{(u)}_{\rm aug})- \gamma{\rm H}(\C^{(v)}_{\rm aug}).
\label{ccloss}
\end{equation}

Second, within each view, we conduct coding consensus learning between the augmented data and the original data. The mechanism behind this is similar to contrastive learning which pulls the positive samples closer. Formally, inspired by weak-to-strong pseudo supervision in semi-supervised learning~\cite{sohn2020fixmatch,xu2024structure}, we transform the posterior class probabilities $Y$ of the original data into hard labels $\hat{T}=[\hat{t}_{ij}]$. $\hat{t}_{ij}^{(v)}$ denotes pseudolabel of the $i$-th data point belonging to the $j$-th class in the $v$-th view. 

The coding consensus is defined by a weak-to-strong pseudo-supervision loss~\cite{xu2024structure}:
\begin{align}
\mathcal L_{\rm Code}=-\sum_{v=1}^{n_v}\sum_{i=1}^n\bigl({\hat{\bm{t}}}^{(v)}_{i}\bigr)^\T\log \y^{(v)}_{i,\rm aug}
\label{w2sloss}
\end{align}
where ${\hat{\bm{t}}}^{(v)}_{i}$ is the one-hot pseudolabel and $\y^{(v)}_{i,\rm aug}$ represent the $i$-th row of $Y^{(v)}_{\rm aug}$. By enforcing the feature of the augmented data consistent with that of the raw data, Eq.~\eqref{w2sloss} requires the encoder to learn the invariant semantics of each view.

Finally, we propose the global consensus to capture global alignment between two views by minimizing the difference between the latent features. The objective is defined by:
\begin{equation}
\mathcal{L}_{\rm Glb}  = -\sum_{\substack{u,v=1\\u> v}}^{n_v}{\rm tr}\Bigl(\bigl({Z^{(u)}}\bigr)^\T {Z}^{(v)}+\bigl({Z_{\rm aug}^{(u)}}\bigr)^\T {Z}_{\rm aug}^{(v)} \Bigr)\,.
\label{gcloss}
\end{equation}
Overall, by minimizing the above inter-view objective, the model is optimized to generate similar features for the same samples between different views. The overall loss objective of the proposed method is given by
\begin{equation}
\mathcal{L}=\mathcal{L}_{\rm Rec}  + \mathcal{L}_{\rm Cls}+\lambda_1\mathcal{L}_{\rm Code} + \lambda_2\mathcal{L}_{\rm Glb}
    \label{total}
\end{equation}
where $\lambda_1$ and $\lambda_2$ denote trade-off hyperparameters.

After the model optimization, the latent embeddings outputted by the encoder of each view are concatenated to obtain the final features for downstream tasks, \ie,
\begin{equation}
Z =  \left [  Z^{(1)},Z^{(2)} ,\ldots, Z^{(n_v)}\right ].
\end{equation}
 The proposed method is summarized in Algorithm~\ref{alg}.
\begin{algorithm}[!t]
\caption{The HCN algorithm.}
\begin{algorithmic}[1]
\REQUIRE Multiview dataset $\mathcal{X}=\{X^{(v)}\}^{n_v}_v$, drop rate $\rho$, hyperparameters $\lambda_1$, $\lambda_2$, $\alpha$, $\beta$, and $\gamma$.
\ENSURE $E$ and $\{\theta^{(v)},\phi^{(v)}, \forall v\in\{1,2,\ldots,n_v\}\}$.
\WHILE{$e<E$}
\FOR{mini-batch samples in $\mathcal{X}$}
\STATE Obtain $X^{(v)}_{\rm aug}$ for each view $v$;
\STATE Compute $Z^{(v)}, Z^{(v)}_{\rm aug}, Y^{(v)}$ and $Y^{(v)}_{\rm aug}$ for each view by the view-specific encoder and the softmax layer;
\STATE Calculate $\mathcal{L}_{\rm Rec}$ and $\mathcal{L}_{\rm Cls}$ by Eqs.~\eqref{recloss} and~\eqref{ccloss};
\STATE Calculate $\mathcal{L}_{\rm Code}$ and $\mathcal{L}_{\rm Glb}$ by Eqs.~\eqref{w2sloss} and~\eqref{gcloss};
\STATE Update $\{\theta^{(v)},\phi^{(v)},\forall\,v\}$ by Eq.~\eqref{total};
\STATE $e = e + 1$;
\ENDFOR
\ENDWHILE
\RETURN Fused feature: $Z =  \left [  Z^{(1)},Z^{(2)} ,\ldots, Z^{(n_v)}\right ]$.
\end{algorithmic}
\label{alg}
\end{algorithm}
\subsection{Computational Complexity}
The complexity of HCN is $\mathcal{O}(nn_v^2k^2b+2nn_vd_vh^{(l+1)})$ where $b$, $h$, $l$, $n$, and $k$ denote the mini-batch size, the maximum number of hidden layers, the layer number, the number of samples, and the feature dimension, respectively. The derivation process is presented in Appendix B.3.

\begin{table*}[!t]
\centering
\begin{tabular*}{0.98\textwidth}{@{\extracolsep{\fill}\,}lcccccccccccc}
\toprule 
\multirow{2}{*}{Method}
&\multicolumn{3}{c}{LandUse-21}  &\multicolumn{3}{c}{Caltech101-20}& \multicolumn{3}{c}{ Scene-15}& \multicolumn{3}{c}{Noisy MNIST}    \\
 \cmidrule{2-13}
& ACC &NMI &ARI&ACC &NMI &ARI&ACC &NMI &ARI&ACC &NMI &ARI\\
\midrule
$\text{SC}_{\text{Agg}}$\textcolor{gray}{\textsubscript{[NeurIPS01]}}& 24.69 & 30.10   &   10.23  &  48.78   &  60.98   &   34.68  &   35.26  &  35.92   & 20.20    &   44.10  &  40.51   &  27.16    \\
 DCCA \textcolor{gray}{\textsubscript{[ICML13]}}&15.51&  23.15&  4.43&  41.89&  59.14&  33.39& 36.18&   38.92&  20.87&  85.53&  89.44&  81.87  \\
 DCCAE \textcolor{gray}{\textsubscript{[ICML15]}}&15.62&  24.41&  4.42&  44.05&  59.12&  34.56&  36.44&   39.78&  21.47&  81.60&  84.69&  70.87  \\
  BMVC\textcolor{gray}{\textsubscript{[TPAMI19]}} &25.34& 28.56&  11.39&  42.55&  63.63&  32.33& 40.50&   41.20& 24.11&  81.27&  76.12&  71.55  \\
 PIC\textcolor{gray}{\textsubscript{[IJCAI20]}} &24.86&  29.74&  10.48&  62.27 &  67.93&  51.56&  38.72&   40.46&  22.12&  -&  -&  -  \\
 $\text{AE}^2\text{Nets}$ \textcolor{gray}{\textsubscript{[CVPR19]}} &24.79& 30.36&  10.35&  49.10&  65.38&  35.66&  36.10&   40.39&  22.08&  56.98&  46.83&  36.98 \\
EERIMVC\textcolor{gray}{\textsubscript{[TPAMI20]}}  &24.92 &29.57 &12.24& 43.28& 55.04 &30.42& 39.60& 38.99& 22.06&  65.47& 57.69 &49.54\\

MFLVC \textcolor{gray}{\textsubscript{[CVPR22]}}&23.67&  27.50&  11.27&  49.25&  41.40&  45.63& 41.49&   42.28&  24.41& 96.91&  92.44&  93.36 \\
CVCL \textcolor{gray}{\textsubscript{[ICCV23]}}&25.40&  29.59&  11.78&  34.77 & 59.93  &25.70 &  38.43&  39.58&  22.53&   \underline{97.87}&  \underline{94.18}   &    \textbf{97.87}  \\
 DCP \textcolor{gray}{\textsubscript{[TPAMI23]}}&26.23& 30.65 &13.70& \underline{70.18} &\underline{68.06} &\underline{76.88}& 41.07 &\underline{45.11} &24.78&  82.78&  84.86&  74.83 \\
   GCFAgg \textcolor{gray}{\textsubscript{[CVPR23]}}&28.06  &32.44  & 14.40 & 34.12 &  53.20 & 23.16  &39.72   &41.37  & 23.01 & 91.44 & 86.56  &83.16\\
   DealMVC\textcolor{gray}{\textsubscript{[MM23]}}& 10.41  &7.11  &1.69& 39.56&  56.91 &36.04  &38.96  & 42.26  &  24.21 & 32.57 & 28.12 & 13.72  \\
  SEM\textcolor{gray}{\textsubscript{[NeurIPS23]}} & \underline{30.02} & \underline{34.75} &\underline{15.93}  & 37.33 & 59.95  &  28.44 &  40.53 & 42.48 & 25.04 & 60.04 & 64.69& 43.15\\
  MVCAN \textcolor{gray}{\textsubscript{[CVPR24]}}&23.94  &29.57  & 10.70 &  48.63& 66.80  & 44.85  & \underline{41.54}  & 44.38 & \underline{25.63} & 78.46 & 79.01  & 70.51  \\

HCN (ours)&\textbf{32.81} &  \textbf{38.58} &  \textbf{17.86}  & \textbf{77.39}    &  \textbf{74.64}    &   \textbf{88.70}    &   \textbf{46.05}  &  \textbf{45.56}    &    \textbf{28.54}   &  \textbf{98.07}     &  \textbf{94.83}     &   \underline{95.79}      \\
\bottomrule
\end{tabular*}
\caption{The clustering results with two views on four datasets. ``-'' indicates unavailable results due to the out-of-memory issue. The best and the second result are \textbf{bold} and \underline{underlined} respectively.}
\label{result1}
\end{table*}

\begin{table*}[t]
\centering
\begin{tabular*}{0.98\textwidth}{@{\extracolsep{\fill}\,}lccccccccc}
\toprule 
\multirow{2}{*}{Method}
&\multicolumn{3}{c}{LandUse-21}  &\multicolumn{3}{c}{Caltech101-20}& \multicolumn{3}{c}{ Scene-15}  \\
 \cmidrule{2-10}
& ACC &NMI &ARI&ACC &NMI &ARI&ACC &NMI &ARI\\
\midrule
$\text{SC}_{\text{Agg}}$\textcolor{gray}{\textsubscript{[NeurIPS01]}}& 26.44 &   32.73 &  11.91   &  48.00   &  60.40   & 33.82    &   34.05  & 34.83    &  19.73  \\ 
MFLVC \textcolor{gray}{\textsubscript{[CVPR22]}} &  22.38&  27.23& 9.91    & 53.43    & 39.33    &  44.06   &   40.15  & 41.41    & 24.11  \\
CVCL \textcolor{gray}{\textsubscript{[ICCV23]}} & 23.47 & 27.50   &  10.61   &36.93     &  56.70  &  26.24   & \underline{44.59}    &   42.17  &24.11  \\
 DCP \textcolor{gray}{\textsubscript{[TPAMI23]}} &  \underline{26.66}& \underline{32.74}   & \underline{13.50}    &\underline{70.58} &   \textbf{69.59}  & \textbf{76.93}    & 41.81    &   \textbf{45.23}  &  25.84 \\
   GCFAgg \textcolor{gray}{\textsubscript{[CVPR23]}} & 24.95 &   29.06 &  10.68   &  35.24  & 56.04    & 24.62    &  44.14   &  43.40   &23.99   \\
   DealMVC\textcolor{gray}{\textsubscript{[MM23]}}  &12.04  &   9.39 & 2.74    &  37.76   &   47.40  &  34.35   &  40.02   &   42.99  &  24.16 \\
  SEM\textcolor{gray}{\textsubscript{[NeurIPS23]}}  &  25.12 &  30.12  & 11.55   &37.34 & 62.51   &  28.66   & 42.45    &  41.25   & \underline{26.67}  \\
  MVCAN \textcolor{gray}{\textsubscript{[CVPR24]}}  & 26.18 &  32.53  &  12.57   &  50.04   &   66.04  &  44.85   & 42.07    &   44.38  &  25.57 \\
HCN (ours) &\textbf{33.30}  &  \textbf{38.55}  & \textbf{18.82}    &   \textbf{71.35}  &   \underline{68.61}  &   \underline{73.25}  &  \textbf{45.20}   &   \underline{44.52}  &  \textbf{28.18}   \\
\bottomrule
\end{tabular*}
\caption{The multiview clustering results on three datasets. The best and the second result are \textbf{bold} and \underline{underlined}, respectively.}
\label{result2}
\end{table*}
\section{Experiments}
\subsection{Experimental Setup}
\textbf{Datasets.} We conduct multiview clustering on four widely used datasets to evaluate the proposed HCN, \ie, (1) LandUse-21~\cite{2010Bag} contains 2,100 satellite images across 21 categories, using PHOG and LBP features as two views; (2) Caltech101-20~\cite{fei2004learning} includes 2,386 RGB images from 20 subjects, with HOG and GIST features as two views; (3) Scene-15~\cite{FeiFei2005} comprises 4,485 images from 15 scene categories, utilizing PHOG and GIST as two views; (4) Noisy MNIST~\cite{wang2015deep} uses the original 70k MNIST images as one view and within-class images with white Gaussian noise as the second view. In three views experiments, we use the HOG, GIST, and LBP features. For the Scene-15 and LandUse-21 datasets for the Caltech101-20 dataset, we use the PHOG, LBP, and GIST features.

\textbf{Comparison Methods.} We compare HCN with the following SOTA methods on multiview clustering. The comparison includes Spectral clustering~\cite{ng2001spectral}, traditional multiview clustering methods such as BMVC~~\cite{Zhang2018}, PIC~\cite{ijcai2019p510} and EERIMVC~\cite{liu2020efficient}, based on CCA or CCA-related methods DCCA~\cite{andrew2013deep}, DCCAE~\cite{wang2015deep} and $\text{AE}^2\text{Nets}$)~\cite{Zhang_2019_CVPR}. Besides, contrastive learning methods are also included, \ie, DCP~\cite{Lin2022}, MFLVC~\cite{Xu_2022_CVPR}, CVCL~\cite{chen2023deep}, GCFAggMVC~\cite{yan2023gcfagg}, DealMVC~\cite{yang2023dealmvc}, and SEM~\cite{xu2024self}. Finally, a more recent method, MVCAN~\cite{xu2024investigating} is employed, which considers the case of noisy views based on deep embedding clustering~\cite{xie2016unsupervised}. 

\begin{table}[!t]
\centering
\begin{tabular}{l|cccccc}
\toprule
\multirow{2}{*}{\begin{tabular}[c]{@{}c@{}}HCN\\ w/o \end{tabular}}  & \multicolumn{3}{c}{LandUse-21 }& \multicolumn{3}{c}{ Noisy MNIST }\\
 \cmidrule{2-7}
\multicolumn{1}{c|}{}&ACC&NMI &ARI & ACC &   NMI &ARI   \\ 
\midrule
$\mathcal{L}_{\rm Rec}$&32.27 & 37.82&17.23 &   97.93  &    94.60  &95.51  \\
 $\mathcal{L}_{\rm Cls}$ & 24.57   &  26.22 &   10.11   &       25.62   &  24.58    & 9.79 \\
 $\mathcal{L}_{\rm Glb}$& 32.21& 38.02 &  17.73 &97.30 &  93.41 & 94.20   \\
 $\mathcal{L}_{\rm Code}$&  32.62  &    38.19  &  17.79  & 97.82  &     94.44& 95.29\\
 DA &   31.95   &  38.06   &   17.63   &   97.19  &    93.11 &   93.96    \\
\midrule
HCN&\textbf{32.81} &  \textbf{38.58} &  \textbf{17.86}  &  \textbf{98.07}     &  \textbf{94.83}     &   \textbf{95.79}      \\
\bottomrule
\end{tabular}
\caption{Ablation study of each component in HCN.}
\label{Ablation}
\end{table}

\textbf{Implementation Details.} We select the best performance for each experiment and report the average performance of five runs with different seeds. $k$-means clustering algorithm is utilized to obtain the clustering results. To have a fair comparison, all the methods use the same views on each dataset, and use three views in multiview setting. In addition, following the DCP~\cite{Lin2022}, we only use a subset of Noisy MNIST consisting of 10k validation images and 10k testing images in the experiments. More details are in Appendix C.1.

\textbf{Evaluation Metrics.} To evaluate the clustering for our HCN, three widely used clustering metrics are employed, \ie, Normalized Mutual Information(NMI), Accuracy (ACC), and Adjusted Rand Index (ARI). For these metrics, a higher numerical value represents a better clustering.

\begin{figure}[!t]
\centering
\includegraphics[width=0.466\textwidth]{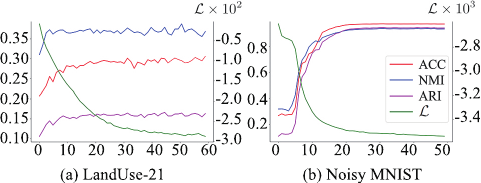}
\caption{Convergence and clustering performance of HCN with increasing epoch on LandUse-21 and Noisy MNIST.}
\label{covergence}
\end{figure}

\begin{figure}[!t]
\centering
\includegraphics[width=0.4\textwidth]{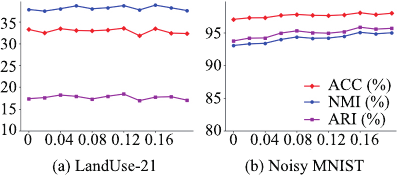}
\caption{Parameter sensitivity of drop rate $\rho$.}
\label{droprate}
\end{figure}
\subsection{Experimental Results}
The clustering results on four datasets are shown in Tables~\ref{result1} and~\ref{result2}, where Table~\ref{result1} reports the performance with two views and Table~\ref{result2} presents the results of three views. From the tables, we have the following observations: (1) Our approach significantly outperforms SOTA baselines on nearly all settings. Remarkably, in Table~\ref{result1}, our approach demonstrates outstanding performance on the Caltech101-20 and Scene-15 datasets, significantly outperforming the best comparison method in ACC and NMI. (2) Compared to contrastive multiview clustering methods, the proposed approach achieves SOTA results. This is attributed to our hierarchical consensus network capturing consensus at different levels, resulting in more comprehensive and discriminative features.
\subsection{Ablation Study}
To validate the importance of each component in HCN, we conducted ablation studies by discarding each component. The results are shown in Table~\ref{Ablation}, where DA means we apply classifying consensus objective without data augmentation. It can be observed that classifying consensus learning plays a more vital role in multiview feature learning than others as the clustering performance is significantly improved by adding classifying consensus objective. The coding consensus and the global consensus objective also improve clustering performance. Furthermore, Table~\ref{Ablation} also demonstrates that data augmentation helps learn more robust, invariant, and discriminative features of the multiview data.
\subsection{Model Analysis}
\textbf{Convergence Analysis.} We investigate the convergence of the proposed method. As illustrated in Figure~\ref{covergence}, we provide the dynamics of loss value and three clustering metrics with increasing epochs on the LandUse-21 and Noisy MNIST datasets. From Figure~\ref{covergence}, we see that the loss value decreases rapidly in a few iterations until convergence is achieved. Meanwhile, the clustering performance metrics quickly increase and stabilize after several epochs. This demonstrates the promising convergence property of HCN.

\textbf{Analysis of Hyperparameters.} Without loss of generality, we investigate the sensitivity of the proposed method w.r.t the trade-off hyperparameters $\lambda_1$ and $\lambda_2$ on the Noisy MNIST and LandUse-21 datasets, where $\lambda_1$ and $\lambda_2$ range from \{0.01, 0.05, 0.1, 0.5, 1, 5\}. As shown in Figure~\ref{lam1lam2}, the accuracy values achieve relatively stable with most combinations of $\lambda_1$ and $\lambda_2$. In addition, Figure~\ref{droprate} presents the sensitivity of HCN w.r.t the drop rate $\rho$ which varies in a range of [0, 0.2] with 0.02 intervals. The results show that the results are relatively stable with $\rho$ varying within a wide range.

\begin{figure}[!t]
\captionsetup[subfloat]{labelsep=none,format=plain,labelformat=empty}
\centering
\subfloat[Raw]{\includegraphics[width=0.12\textwidth]{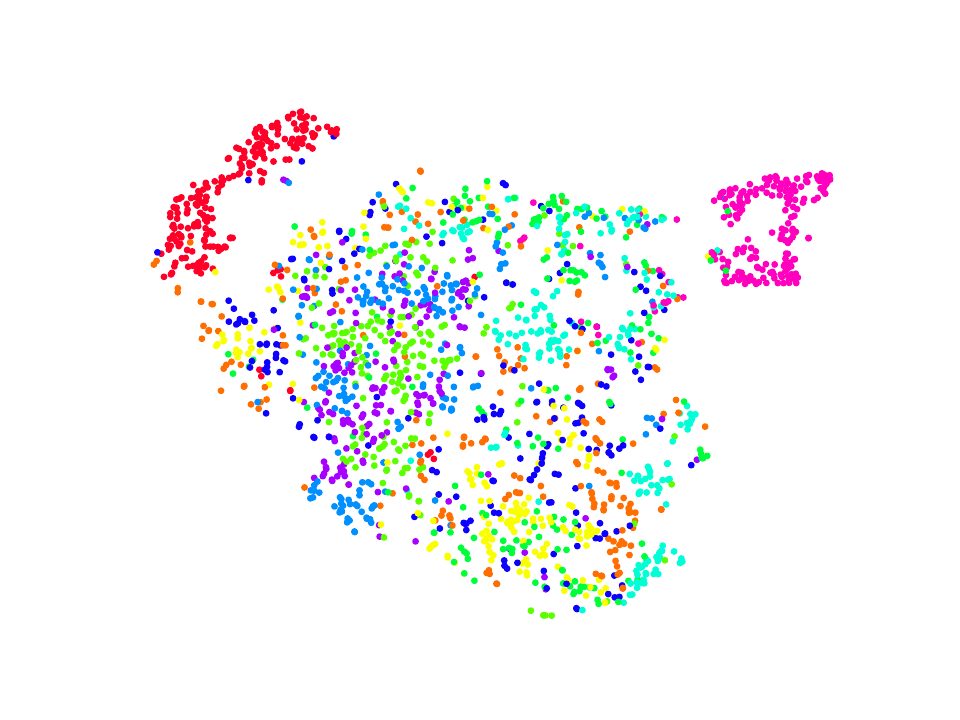}}
\subfloat[CVCL]{\includegraphics[width=0.12\textwidth]{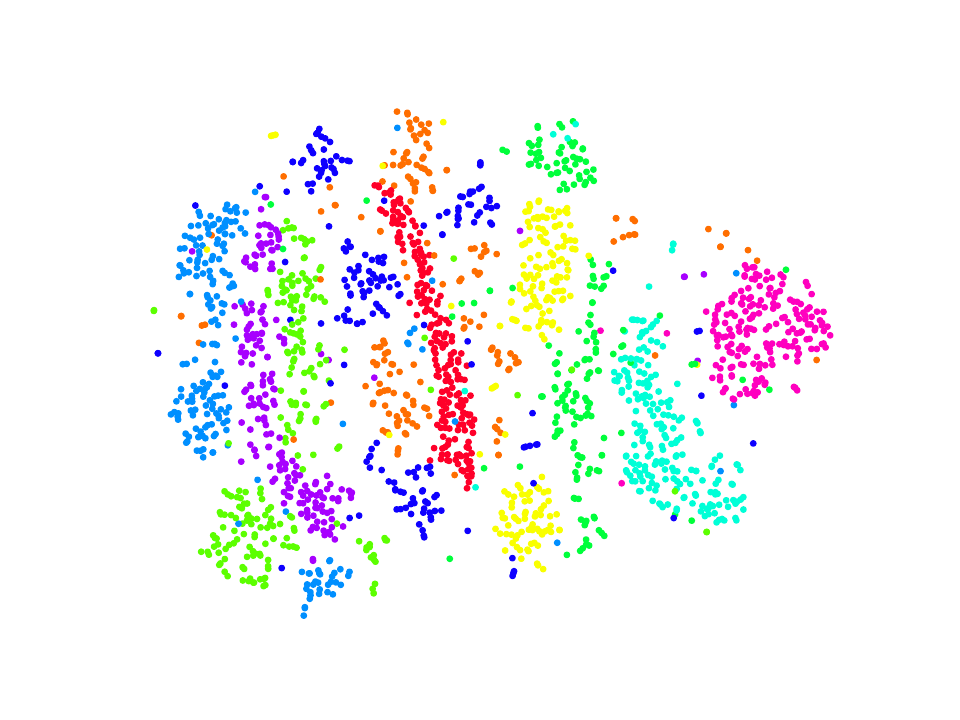}}
\subfloat[MVCAN]
{\includegraphics[width=0.12\textwidth]{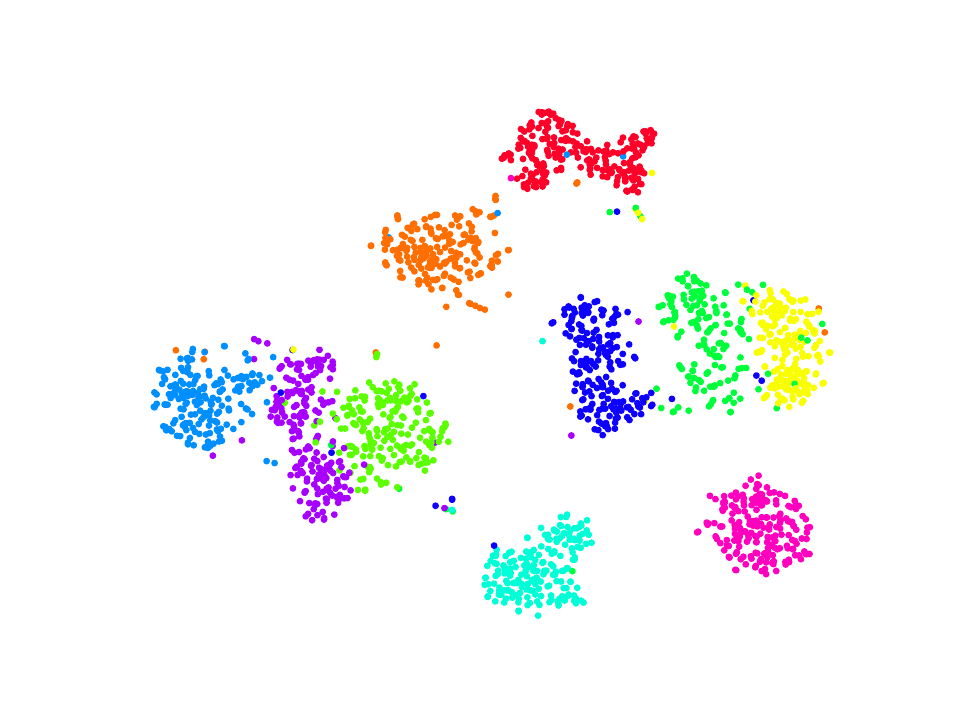}}
\subfloat[HCN]{\includegraphics[width=0.12\textwidth]{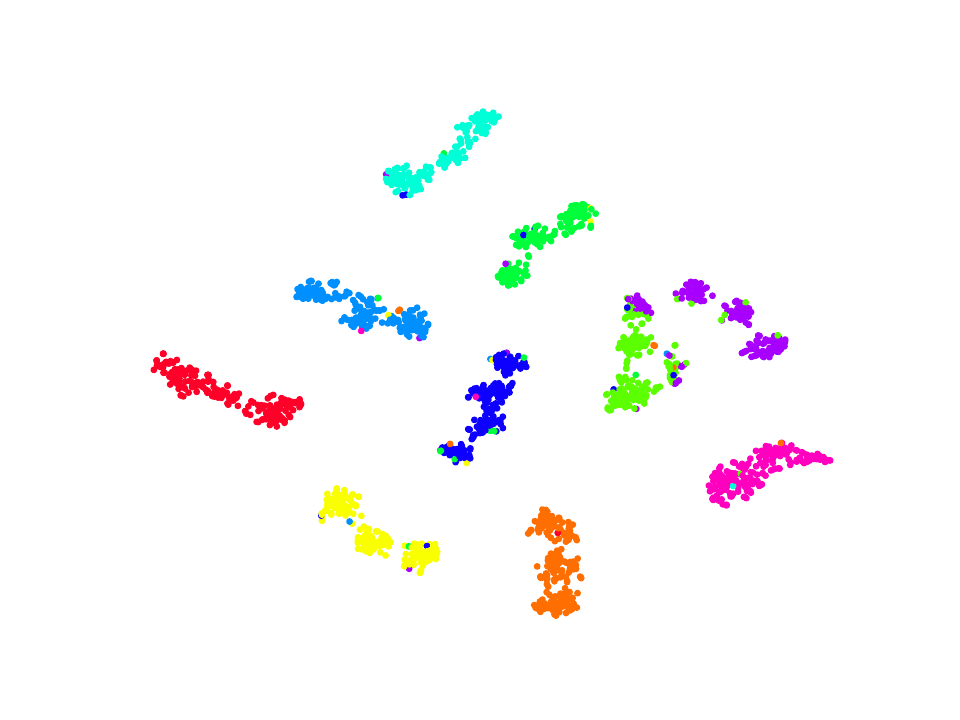}}
\caption{Visualizations on Noisy MNIST with baselines.}
\label{visul}
\end{figure}
\begin{figure}[!t]
\centering
\includegraphics[width=0.46\textwidth]{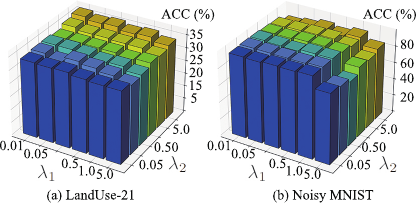}
\caption{Parameter sensitivity of $\lambda_1$ and $\lambda_2$.}
\label{lam1lam2}
\end{figure}
\subsection{Visualization}
In Figure~\ref{visul}, we provide the $t$-SNE visualizations of our approach and other baselines on Noisy MNIST. Specifically, we visualize the fused features learned by the encoder of the proposed HCN. From the figure, we observe that raw features are non-discriminative at the initial stage. After training with HCN, the learned features become more discriminative among different classes and each class becomes more compact compared to baselines, demonstrating the effectiveness of HCN for multiview feature learning. The visualizations of other comparison methods are in Appendix C.3.
\section{Conclusion}
In this paper, we explore hierarchical consensus learning for multiview feature learning and introduce three consensus indices across views, \ie, classifying, coding, and global consensus, providing promising insights into multiview feature learning. Specifically, classifying consensus explores consistency between views from a CCA perspective, while the coding consensus closely resembles contrastive learning. Global consensus simultaneously captures both the coding and classifying effects. Based on these, we propose HCN for multiview feature learning. HCN effectively captures hierarchical consensus between different views from different perspectives, resulting in more comprehensive and discriminative features. Extensive experiments and ablation studies on four multiview datasets validate the superiority and effectiveness of our HCN method. 
\section*{Acknowledgments}
This work was supported by the National Natural Science Foundation of China under Grant No.~62176108, Natural Science Foundation of Qinghai Province of China under No.~2022-ZJ-929, Science Foundation of National Archives Administration of China under No.~2024-B-006, and Supercomputing Center of Lanzhou University.

\bibliography{aaai25}

\begin{thebibliography}{36}
\providecommand{\natexlab}[1]{#1}

\bibitem[{Andrew et~al.(2013)Andrew, Arora, Bilmes, and
  Livescu}]{andrew2013deep}
Andrew, G.; Arora, R.; Bilmes, J.; and Livescu, K. 2013.
\newblock Deep canonical correlation analysis.
\newblock In \emph{ICML}, 1247--1255.

\bibitem[{Bishop(2006)}]{BishopPRML}
Bishop. 2006.
\newblock \emph{Pattern Recognition and Machine Learning}.
\newblock Springer.

\bibitem[{Blum and Mitchell(1998)}]{blum1998combining}
Blum, A.; and Mitchell, T. 1998.
\newblock Combining labeled and unlabeled data with co-training.
\newblock In \emph{COLT}, 92--100.

\bibitem[{Chen et~al.(2023)Chen, Mao, Woo, and Peng}]{chen2023deep}
Chen, J.; Mao, H.; Woo, W.~L.; and Peng, X. 2023.
\newblock Deep multiview clustering by contrasting cluster assignments.
\newblock In \emph{ICCV}, 16752--16761.

\bibitem[{Chen et~al.(2021)Chen, Yang, Mao, and Fahy}]{chen2021multiview}
Chen, J.; Yang, S.; Mao, H.; and Fahy, C. 2021.
\newblock Multiview subspace clustering using low-rank representation.
\newblock \emph{TCyb}, 52(11): 12364--12378.

\bibitem[{Chen et~al.(2020)Chen, Kornblith, Norouzi, and
  Hinton}]{chen2020simple}
Chen, T.; Kornblith, S.; Norouzi, M.; and Hinton, G. 2020.
\newblock A simple framework for contrastive learning of visual
  representations.
\newblock In \emph{ICML}, 1597--1607.

\bibitem[{Chen and He(2021)}]{chen2021exploring}
Chen, X.; and He, K. 2021.
\newblock Exploring simple siamese representation learning.
\newblock In \emph{CVPR}, 15750--15758.

\bibitem[{Fei-Fei, Fergus, and Perona(2004)}]{fei2004learning}
Fei-Fei, L.; Fergus, R.; and Perona, P. 2004.
\newblock Learning generative visual models from few training examples: An
  incremental bayesian approach tested on 101 object categories.
\newblock In \emph{CVPR workshop}, 178--178.

\bibitem[{Fei-Fei and Perona(2005)}]{FeiFei2005}
Fei-Fei, L.; and Perona, P. 2005.
\newblock A bayesian hierarchical model for learning natural scene categories.
\newblock In \emph{CVPR}, volume~2, 524--531.

\bibitem[{Grill et~al.(2020)Grill, Strub, Altch{\'e}, Tallec, Richemond,
  Buchatskaya, Doersch, Avila~Pires, Guo, Gheshlaghi~Azar
  et~al.}]{grill2020bootstrap}
Grill, J.-B.; Strub, F.; Altch{\'e}, F.; Tallec, C.; Richemond, P.;
  Buchatskaya, E.; Doersch, C.; Avila~Pires, B.; Guo, Z.; Gheshlaghi~Azar, M.;
  et~al. 2020.
\newblock Bootstrap your own latent-a new approach to self-supervised learning.
\newblock In \emph{NeurIPS}, volume~33, 21271--21284.

\bibitem[{Han et~al.(2024)Han, Tian, Xia, and Zhan}]{Hanijcai2024}
Han, Q.; Tian, Z.; Xia, C.; and Zhan, K. 2024.
\newblock {InfoMatch}: Entropy neural estimation for semi-supervised image
  classification.
\newblock In \emph{IJCAI}, volume~33, 4089--4097.

\bibitem[{Haussler(1999)}]{Haussler1999}
Haussler, D. 1999.
\newblock Convolution kernels on discrete structures.
\newblock Technical report.

\bibitem[{He et~al.(2020)He, Fan, Wu, Xie, and Girshick}]{he2020momentum}
He, K.; Fan, H.; Wu, Y.; Xie, S.; and Girshick, R. 2020.
\newblock Momentum contrast for unsupervised visual representation learning.
\newblock In \emph{CVPR}, 9729--9738.

\bibitem[{Hotelling(1936)}]{hotelling1936relations}
Hotelling, H. 1936.
\newblock Relations between two sets of variates.
\newblock \emph{Biometrika}.

\bibitem[{Ji, Henriques, and Vedaldi(2019)}]{ji2019invariant}
Ji, X.; Henriques, J.~F.; and Vedaldi, A. 2019.
\newblock Invariant information clustering for unsupervised image
  classification and segmentation.
\newblock In \emph{ICCV}, 9865--9874.

\bibitem[{Li et~al.(2021)Li, Hu, Liu, Peng, Zhou, and Peng}]{li2021contrastive}
Li, Y.; Hu, P.; Liu, Z.; Peng, D.; Zhou, J.~T.; and Peng, X. 2021.
\newblock Contrastive clustering.
\newblock In \emph{AAAI}, volume~35, 8547--8555.

\bibitem[{Lin et~al.(2023)Lin, Gou, Liu, Bai, Lv, and Peng}]{Lin2022}
Lin, Y.; Gou, Y.; Liu, X.; Bai, J.; Lv, J.; and Peng, X. 2023.
\newblock Dual contrastive prediction for incomplete multi-view representation
  learning.
\newblock \emph{TPAMI}, 45(4): 4447--4461.

\bibitem[{Lin et~al.(2021)Lin, Gou, Liu, Li, Lv, and Peng}]{lin2021completer}
Lin, Y.; Gou, Y.; Liu, Z.; Li, B.; Lv, J.; and Peng, X. 2021.
\newblock COMPLETER: Incomplete multi-view clustering via contrastive
  prediction.
\newblock In \emph{CVPR}, 11174--11183.

\bibitem[{Liu et~al.(2020)Liu, Li, Tang, Xia, Xiong, Liu, Kloft, and
  Zhu}]{liu2020efficient}
Liu, X.; Li, M.; Tang, C.; Xia, J.; Xiong, J.; Liu, L.; Kloft, M.; and Zhu, E.
  2020.
\newblock Efficient and effective regularized incomplete multi-view clustering.
\newblock \emph{TPAMI}, 43(8): 2634--2646.

\bibitem[{Ng, Jordan, and Weiss(2001)}]{ng2001spectral}
Ng, A.; Jordan, M.; and Weiss, Y. 2001.
\newblock On spectral clustering: Analysis and an algorithm.
\newblock In \emph{NeurIPS}, volume~14.

\bibitem[{Sohn et~al.(2020)Sohn, Berthelot, Carlini, Zhang, Zhang, Raffel,
  Cubuk, Kurakin, and Li}]{sohn2020fixmatch}
Sohn, K.; Berthelot, D.; Carlini, N.; Zhang, Z.; Zhang, H.; Raffel, C.~A.;
  Cubuk, E.~D.; Kurakin, A.; and Li, C.-L. 2020.
\newblock Fixmatch: Simplifying semi-supervised learning with consistency and
  confidence.
\newblock In \emph{NeurIPS}, volume~33, 596--608.

\bibitem[{Tian, Chen, and Ganguli(2021)}]{Tian2021}
Tian, Y.; Chen, X.; and Ganguli, S. 2021.
\newblock Understanding self-supervised learning dynamics without contrastive
  pairs.
\newblock In \emph{ICML}, 10268--10278.

\bibitem[{Trosten et~al.(2021)Trosten, Lokse, Jenssen, and
  Kampffmeyer}]{trostenMVC}
Trosten, D.~J.; Lokse, S.; Jenssen, R.; and Kampffmeyer, M. 2021.
\newblock Reconsidering representation alignment for multi-view clustering.
\newblock In \emph{CVPR}, 1255--1265.

\bibitem[{Wang et~al.(2019)Wang, Zong, Liu, Yang, and Zhou}]{ijcai2019p510}
Wang, H.; Zong, L.; Liu, B.; Yang, Y.; and Zhou, W. 2019.
\newblock Spectral Perturbation Meets Incomplete Multi-view Data.
\newblock In \emph{IJCAI}, 3677--3683.

\bibitem[{Wang et~al.(2015)Wang, Arora, Livescu, and Bilmes}]{wang2015deep}
Wang, W.; Arora, R.; Livescu, K.; and Bilmes, J. 2015.
\newblock On deep multi-view representation learning.
\newblock In \emph{ICML}, 1083--1092.

\bibitem[{Xie, Girshick, and Farhadi(2016)}]{xie2016unsupervised}
Xie, J.; Girshick, R.; and Farhadi, A. 2016.
\newblock Unsupervised deep embedding for clustering analysis.
\newblock In \emph{ICML}, 478--487.

\bibitem[{Xu et~al.(2024{\natexlab{a}})Xu, Chen, Ren, Shi, Shen, Niu, and
  Zhu}]{xu2024self}
Xu, J.; Chen, S.; Ren, Y.; Shi, X.; Shen, H.; Niu, G.; and Zhu, X.
  2024{\natexlab{a}}.
\newblock Self-weighted contrastive learning among multiple views for
  mitigating representation degeneration.
\newblock In \emph{NeurIPS}, volume~36.

\bibitem[{Xu et~al.(2024{\natexlab{b}})Xu, Ren, Wang, Feng, Zhang, Niu, and
  Zhu}]{xu2024investigating}
Xu, J.; Ren, Y.; Wang, X.; Feng, L.; Zhang, Z.; Niu, G.; and Zhu, X.
  2024{\natexlab{b}}.
\newblock Investigating and mitigating the side effects of noisy views for
  self-supervised clustering algorithms in practical multi-view scenarios.
\newblock In \emph{CVPR}, 22957--22966.

\bibitem[{Xu et~al.(2022)Xu, Tang, Ren, Peng, Zhu, and He}]{Xu_2022_CVPR}
Xu, J.; Tang, H.; Ren, Y.; Peng, L.; Zhu, X.; and He, L. 2022.
\newblock Multi-Level feature learning for contrastive multi-View clustering.
\newblock In \emph{CVPR}, 16051--16060.

\bibitem[{Xu et~al.(2024{\natexlab{c}})Xu, Zhang, Zhang, and
  Zhan}]{xu2024structure}
Xu, S.; Zhang, X.; Zhang, P.; and Zhan, K. 2024{\natexlab{c}}.
\newblock Structure-Aware Consensus Network on Graphs with Few Labeled Nodes.
\newblock \emph{arXiv}.

\bibitem[{Yan et~al.(2023)Yan, Zhang, Lv, Tang, Yue, Liao, and
  Lin}]{yan2023gcfagg}
Yan, W.; Zhang, Y.; Lv, C.; Tang, C.; Yue, G.; Liao, L.; and Lin, W. 2023.
\newblock {GCFAgg}: Global and cross-view feature aggregation for multi-view
  clustering.
\newblock In \emph{CVPR}, 19863--19872.

\bibitem[{Yang et~al.(2023)Yang, Jiaqi, Wang, Liang, Liu, Wen, Liu, Zhou, Liu,
  and Zhu}]{yang2023dealmvc}
Yang, X.; Jiaqi, J.; Wang, S.; Liang, K.; Liu, Y.; Wen, Y.; Liu, S.; Zhou, S.;
  Liu, X.; and Zhu, E. 2023.
\newblock Dealmvc: Dual contrastive calibration for multi-view clustering.
\newblock In \emph{ACM Multimedia}, 337--346.

\bibitem[{Yang and Newsam(2010)}]{2010Bag}
Yang, Y.; and Newsam, S. 2010.
\newblock Bag-of-visual-words and spatial extensions for land-use
  classification.
\newblock In \emph{SIGSPATIAL}.

\bibitem[{Zhan et~al.(2018)Zhan, Nie, Wang, and Yang}]{zhan2018multiview}
Zhan, K.; Nie, F.; Wang, J.; and Yang, Y. 2018.
\newblock Multiview consensus graph clustering.
\newblock \emph{TIP}, 28(3): 1261--1270.

\bibitem[{Zhang, Liu, and Fu(2019)}]{Zhang_2019_CVPR}
Zhang, C.; Liu, Y.; and Fu, H. 2019.
\newblock {AE$^2$-Nets}: Autoencoder in autoencoder networks.
\newblock In \emph{CVPR}, 2577--2585.

\bibitem[{Zhang et~al.(2018)Zhang, Liu, Shen, Shen, and Shao}]{Zhang2018}
Zhang, Z.; Liu, L.; Shen, F.; Shen, H.~T.; and Shao, L. 2018.
\newblock Binary multi-view clustering.
\newblock \emph{TPAMI}, 41(7): 1774--1782.

\end{thebibliography}
\clearpage
\appendix
\section{Dataset}
We conduct multiview clustering experiments on four widely used multiview datasets as shown in Table~\ref{dataset}, which are summarized as follows:
\begin{itemize}
\item LandUse-21~\cite{2010Bag} comprises 2,100 satellite images spanning 21 categories, offering three distinct views. We employ PHOG and LBP features as two views.
\item Caltech101-20~\cite{fei2004learning}, an RGB image dataset with multiple views, includes 2,386 images featuring 20 subjects. Six distinct features are extracted, and we utilize HOG and GIST features as two views.
\item Scene-15~\cite{FeiFei2005}, totaling 4,485 images across 15 scene categories in indoor and outdoor settings, extracts three features (GIST, PHOG, and LBP) for each image, with PHOG and GIST used as two views.
\item Noisy MNIST~\cite{wang2015deep} employs the original 70k MNIST images as one view and randomly selects within-class images with white Gaussian noise as the other view.
\end{itemize}

\begin{table}[h]
\centering
\begin{tabular*}{0.48\textwidth}{@{\extracolsep{\fill}}l|rrrl}
\toprule 
 Datasets&  $n$~~~~&$n_v$&$k$& dimensionalities \\
 \midrule
Caltech101-20  &  2386&6&20&\small{48,40,254,1984,512,928}\\
Scene-15&4485&3& 15&\small{20,59,40}\\
LandUse-21&2100&3&21&\small{20,59,40}\\
Noisy MNIST&70000&2&10&\small{784,784}\\
\bottomrule
\end{tabular*}
\caption{The detail of the datasets }
\label{dataset}
\end{table}
\section{Extended derivations}
\subsection{Derivation of Eq. (6)}
Below is a derivation of Eq. (6). According to the information theory, for $\C^{(i)}$ and $\C^{(j)}$, the joint entropy is given by:
\begin{equation*}
{\rm H}(\C^{(i)},\C^{(j)})  ={\rm H} (\C^{(j)}|\C^{(i)})+ {\rm H} (\C^{(i)})
\end{equation*}
Thus, for the two views, we have
\begin{equation}
{\rm H} (\C^{(2)}|\C^{(1)})={\rm H}(\C^{(2)})-{\rm H} (\C^{(1)})+{\rm H} (\C^{(1)}|\C^{(2)}). \tag{B1}   
\end{equation}
By substituting Eq.~(B1) into  Eq.~(4) and combining Eq.~(5), we have Eq.~(6). 
\begin{table*}[!t]
    \centering
    \resizebox{0.95\textwidth}{!}{\begin{tabular}{ccccccccccccccccc}
\toprule 
\multirow{2}{*}{$\mathcal{L}_{\rm Rec}$}  & \multirow{2}{*}{$\mathcal{L}_{\rm Cls}$}  & \multirow{2}{*}{$\mathcal{L}_{\rm Glb}$ } & \multirow{2}{*}{$\mathcal{L}_{\rm Code}$}& \multirow{2}{*}{DA} & \multicolumn{3}{c}{LandUse-21 }& \multicolumn{3}{c}{Caltech101-20} & \multicolumn{3}{c}{ Scene-15}& \multicolumn{3}{c}{ Noisy MNIST }\\
 \cmidrule{6-17}
\multicolumn{5}{c}{}&ACC&NMI &ARI& ACC  &  NMI  &    ARI  & ACC &  NMI  &     ARI & ACC &   NMI &ARI   \\ 
\midrule
&\Checkmark &\Checkmark &  \Checkmark &\Checkmark&32.27 & 37.82&17.23 &  58.30  &  69.54    & 54.20  &  44.82  &  44.78    &  27.98   &  97.93  &    94.60  &95.51  \\
\Checkmark&    &\Checkmark &  \Checkmark & \Checkmark&24.57   &  26.22 &   10.11   &      61.35 &  60.54 &  66.47   &   34.04 &  36.16    &   19.48  & 25.62   &  24.58    & 9.79 \\
\Checkmark&\Checkmark & & \Checkmark&\Checkmark& 32.21& 38.02 &  17.73 & 62.80  & 69.37 &61.93   & 45.43  & 45.16 &  27.42 &97.30 &  93.41 & 94.20   \\
\Checkmark&\Checkmark & \Checkmark &     & \Checkmark& 32.62  &    38.19  &  17.79  & 77.26     &   74.56   & 88.51  &   40.60   &   43.79  & 26.30   & 97.82  &     94.44& 95.29\\
\Checkmark&\Checkmark &\Checkmark & \Checkmark &  & 31.95   &  38.06   &   17.63 &  74.64   & 73.90   &  84.73  & 42.17 &  46.45 &   28.39  & 97.19  &    93.11 &   93.96 \\
\Checkmark&\Checkmark &\Checkmark & \Checkmark & \Checkmark &\textbf{32.81} &  \textbf{38.58} &  \textbf{17.86}  & \textbf{77.39}    &  \textbf{74.64}    &   \textbf{88.70}    &   \textbf{46.05}  &  45.56    &    \textbf{28.54}&  \textbf{98.07}     &  \textbf{94.83}     &   \textbf{95.79}      \\
\bottomrule
\end{tabular}}
\caption{Ablation study of each loss component in the proposed method on all the datasets. 
}
\label{Ablation}
\end{table*}

\begin{table*}[!t]
\centering
\resizebox{0.95\textwidth}{!}{\begin{tabular}{ccccccccccccccccc}
\toprule 

\multirow{2}{*}{$\mathcal{L}_{\rm Rec}$}  & \multirow{2}{*}{$\mathcal{L}_{\rm Cls}$}  & \multirow{2}{*}{$\mathcal{L}_{\rm Glb}$ } & \multirow{2}{*}{$\mathcal{L}_{\rm Code}$}& \multirow{2}{*}{DA} & \multicolumn{3}{c}{LandUse-21 }& \multicolumn{3}{c}{Caltech101-20} & \multicolumn{3}{c}{ Scene-15}\\
 \cmidrule{6-14}
\multicolumn{5}{c}{}&ACC&NMI &ARI& ACC  &  NMI  &    ARI  & ACC &  NMI  &     ARI   \\ 
\midrule
&\Checkmark &\Checkmark &  \Checkmark &\Checkmark&31.85 & 36.96&16.91  & 66.93    & 65.31  &  70.07  & 44.41    & 43.57  & 27.43    \\
\Checkmark&    &\Checkmark &  \Checkmark & \Checkmark&27.03& 30.72 &  12.19 & 49.31  &50.19 &44.22  & 32.49  & 33.97 &  16.59  \\
\Checkmark&\Checkmark & & \Checkmark&\Checkmark&  31.15  &    36.43 & 16.70  &55.71    &  63.03 &48.64 &   43.32 &  43.41 & 26.68    \\
\Checkmark&\Checkmark & \Checkmark &   & \Checkmark  & 30.03  & 35.01  &  15.29 & 67.24 &  66.56&   70.28& 40.93  &  41.42 & 24.30 \\
\Checkmark&\Checkmark &\Checkmark & \Checkmark &  & 31.55   &  36.75  &  16.76 &  68.62  & 67.19  & 72.59 & 40.31 &  43.38 &   26.63   \\
\Checkmark&\Checkmark &\Checkmark & \Checkmark & \Checkmark  &\textbf{33.30}  &  \textbf{38.55}  & \textbf{18.82}    &   \textbf{71.24}  &   \textbf{68.49}  &   \textbf{73.25}  &  \textbf{45.20}   &   \textbf{44.52}  &  \textbf{28.18}      \\
\bottomrule
\end{tabular}}
\caption{Ablation study of each loss component in the proposed method with three views on all the datasets.}
\label{Ablation_3}
\end{table*}

\begin{figure*}[!t]
\captionsetup[subfloat]{labelsep=none,format=plain,labelformat=empty}
\centering
\subfloat{\includegraphics[width=0.13\textwidth]{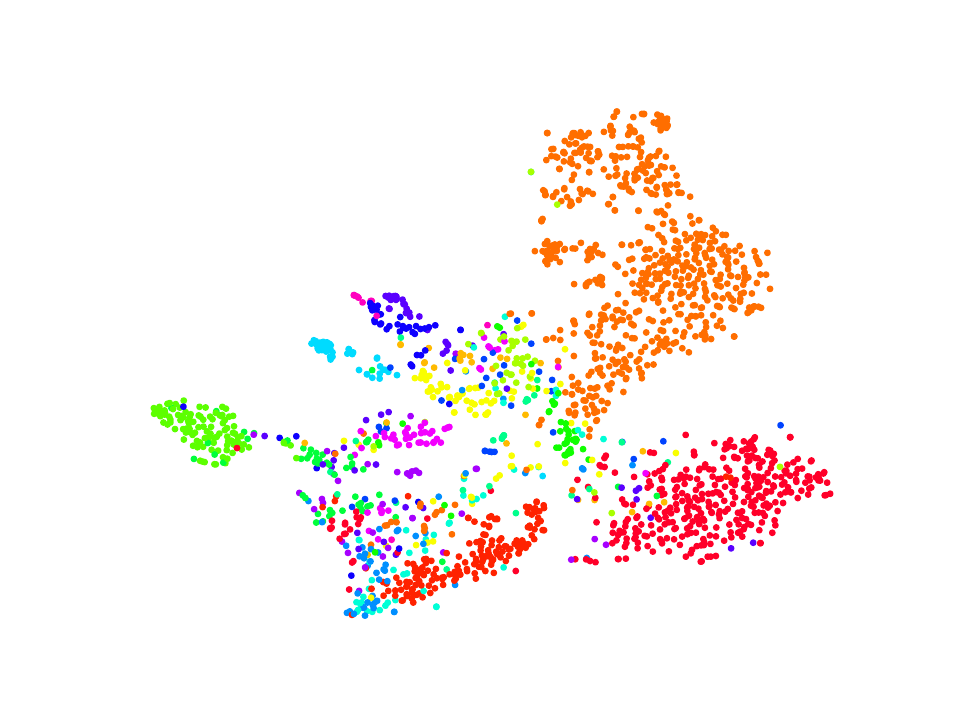}}
\subfloat{\includegraphics[width=0.13\textwidth]{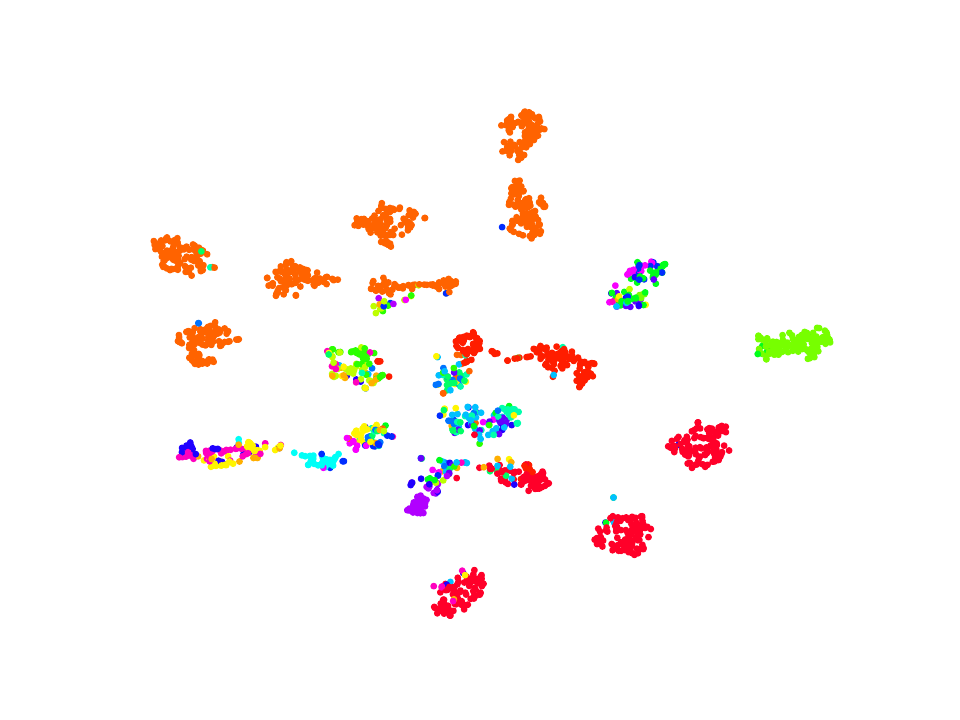}}
\subfloat{\includegraphics[width=0.13\textwidth]{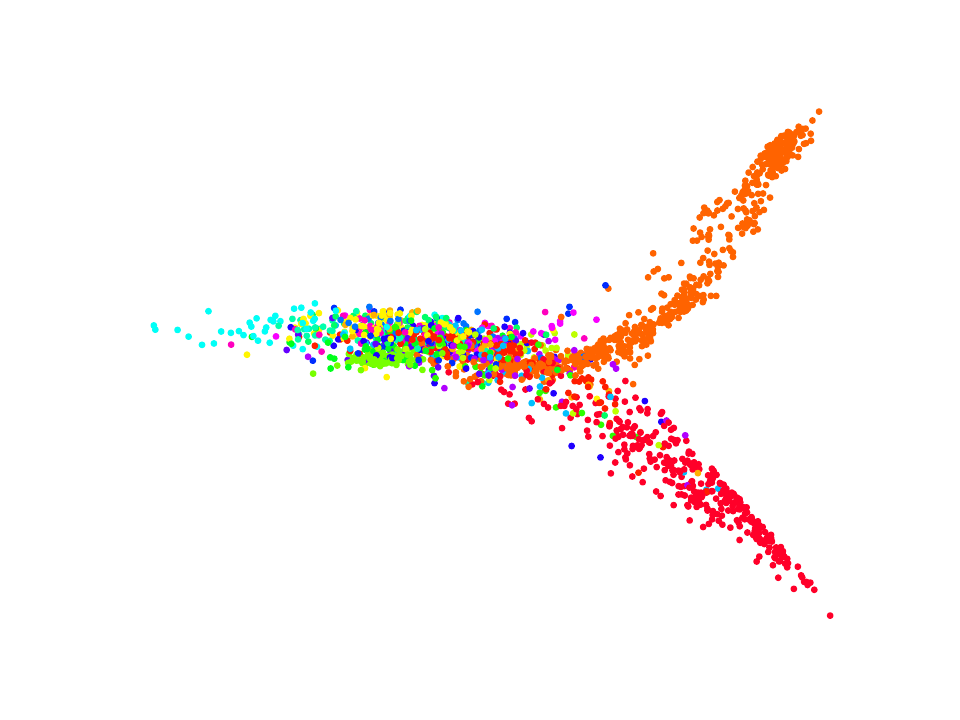}}
\subfloat{\includegraphics[width=0.13\textwidth]{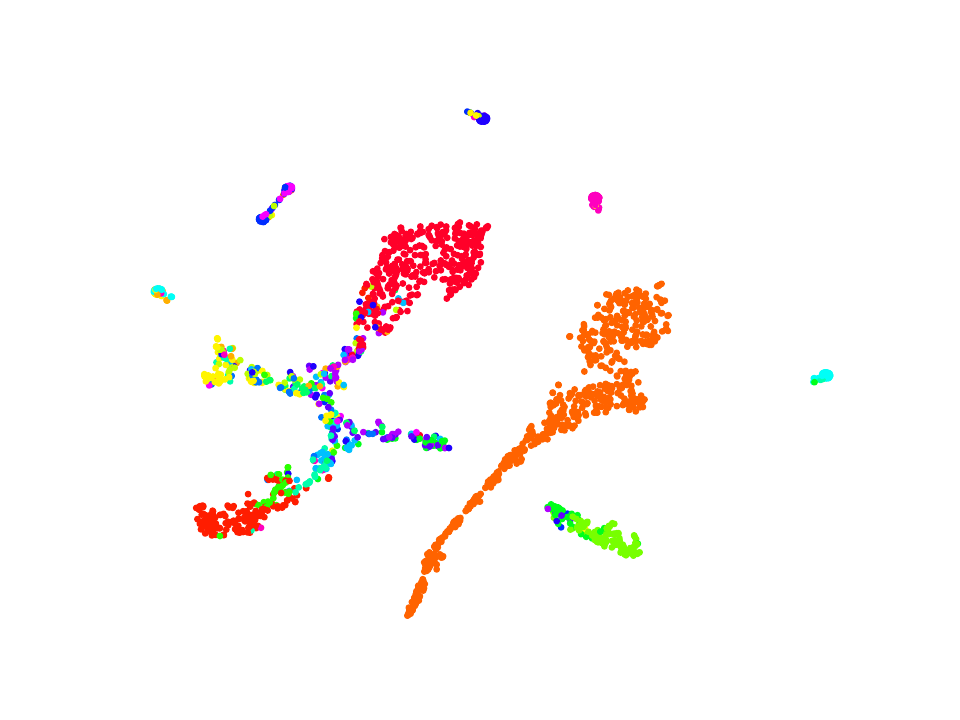}}
\subfloat{\includegraphics[width=0.13\textwidth]{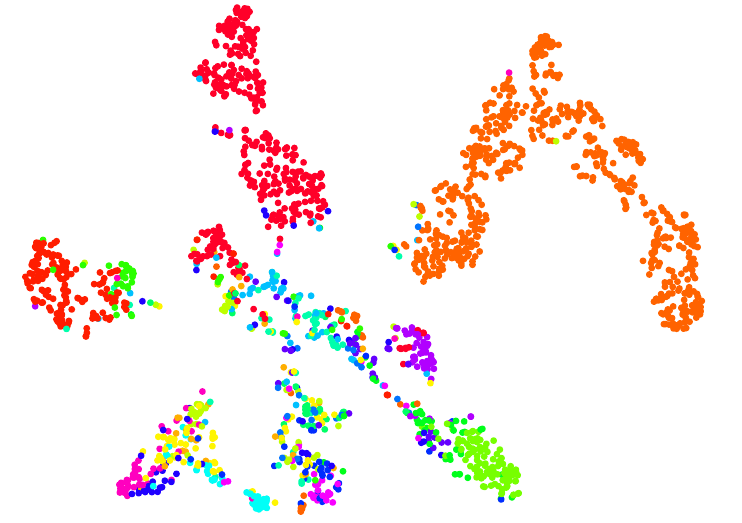}}
\subfloat{\includegraphics[width=0.13\textwidth]{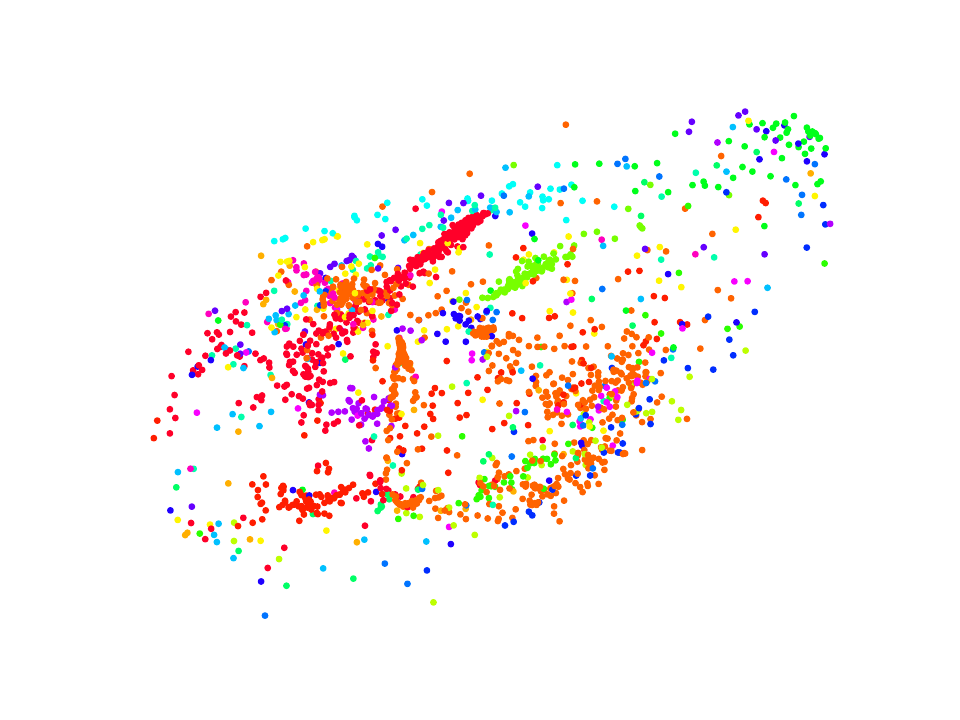}}
\subfloat{\includegraphics[width=0.13\textwidth]{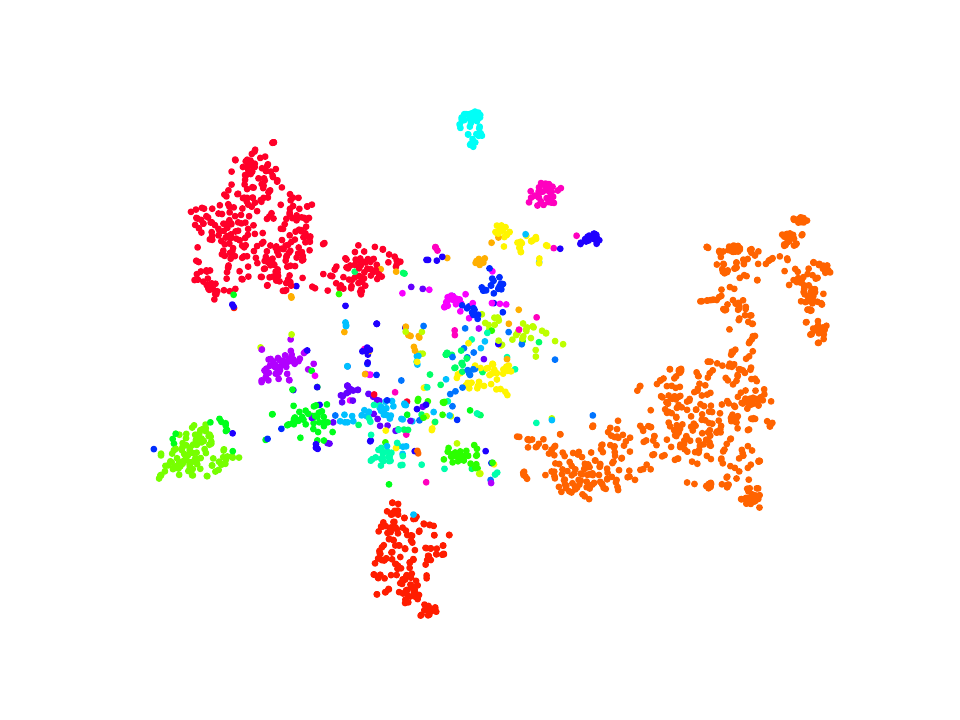}}
\subfloat{\includegraphics[width=0.13\textwidth]{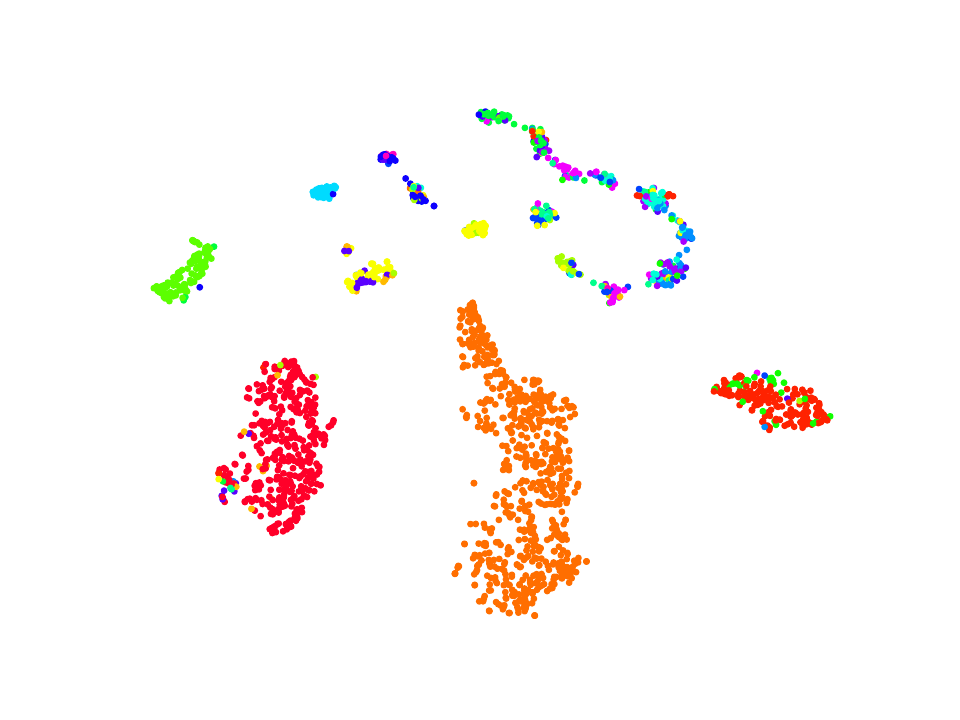}}
\quad
\subfloat[Raw]{\includegraphics[width=0.13\textwidth]{fig/visualization/noisy_0.pdf}}
\subfloat[CVCL]{\includegraphics[width=0.13\textwidth]{fig/visualization/CVCL_3.pdf}}
\subfloat[MFLVC]{\includegraphics[width=0.13\textwidth]{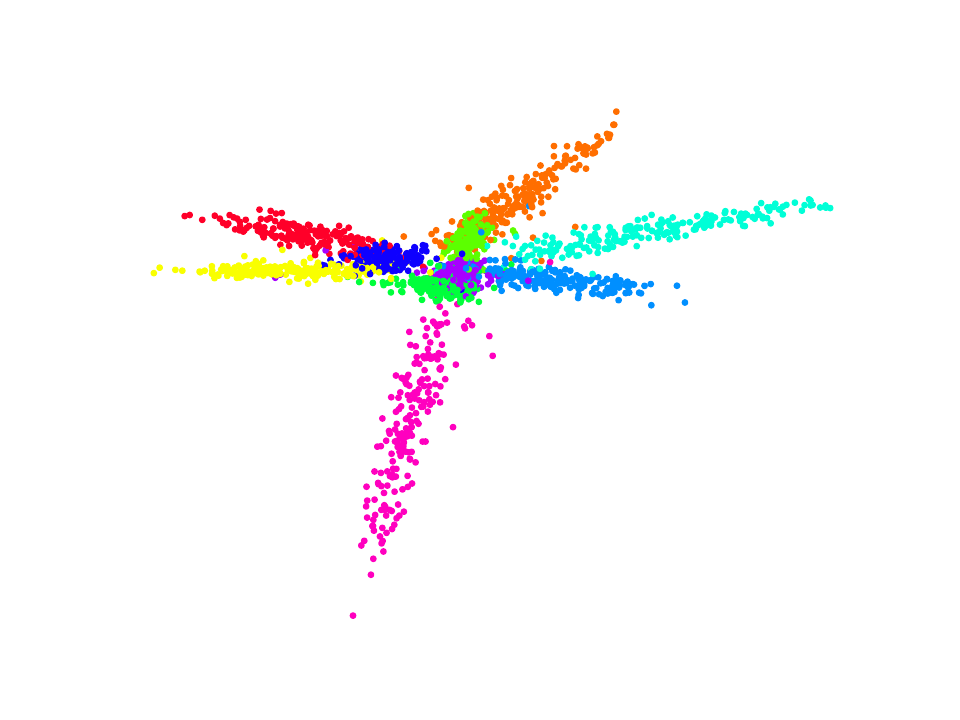}}
\subfloat[DCP]{\includegraphics[width=0.13\textwidth]{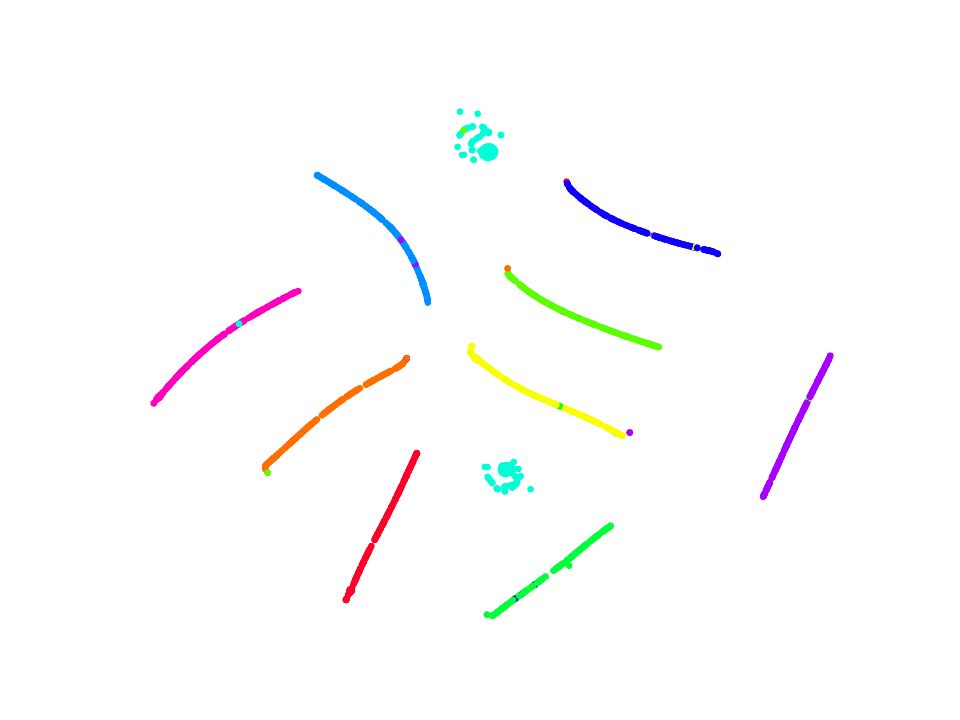}}
\subfloat[DealMVC]
{\includegraphics[width=0.13\textwidth]{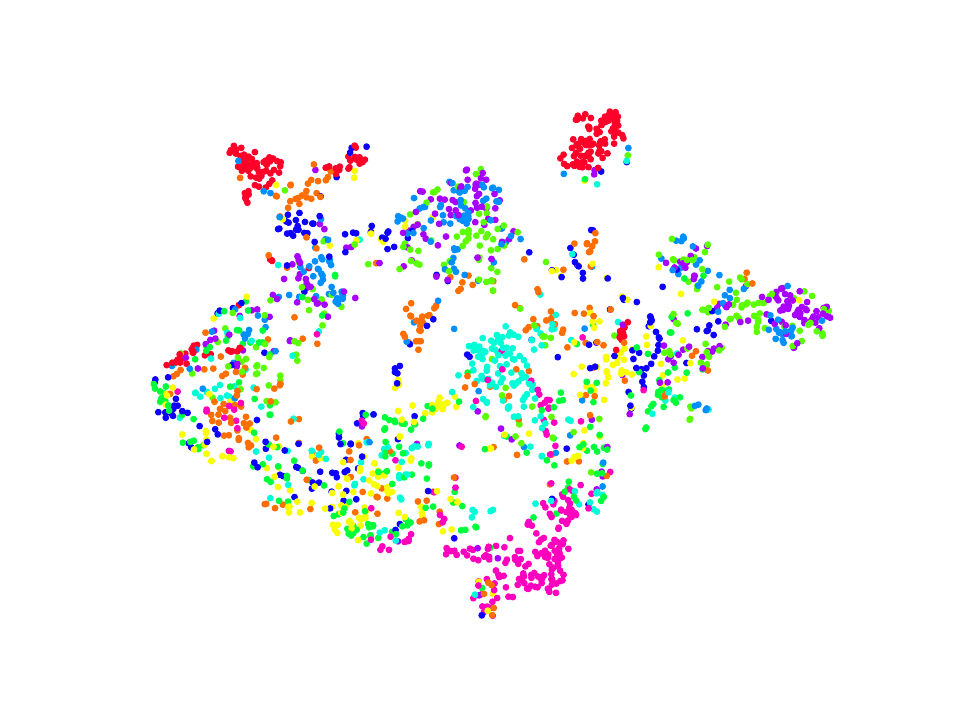}}
\subfloat[SEM]
{\includegraphics[width=0.13\textwidth]{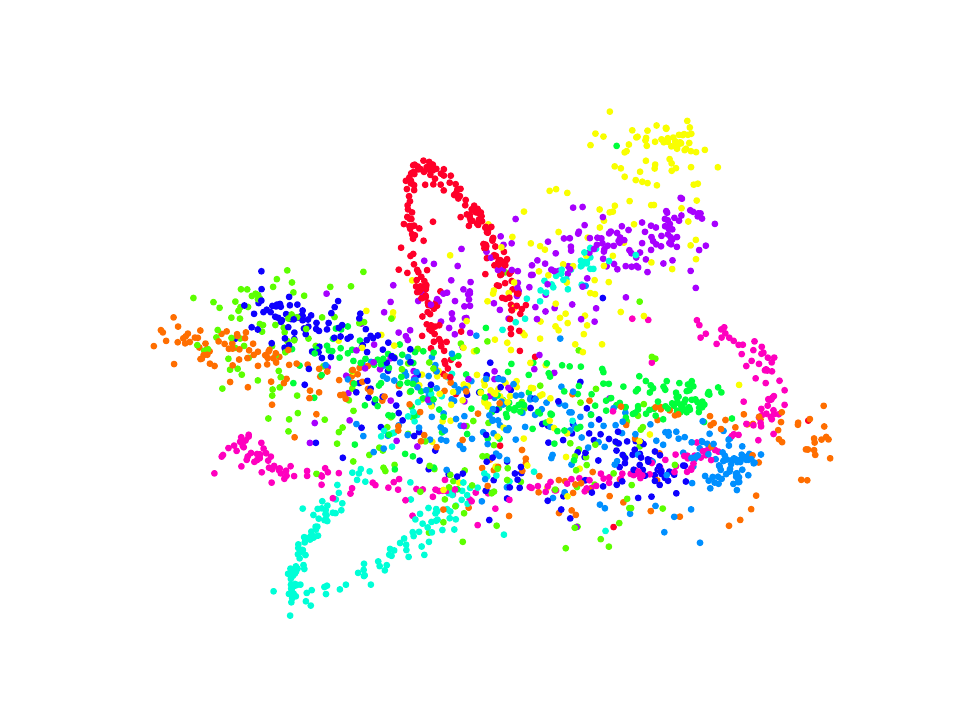}}
\subfloat[MVCAN]
{\includegraphics[width=0.13\textwidth]{fig/visualization/view_fuse_MVCAN3.pdf}}
\subfloat[HCN]{\includegraphics[width=0.15\textwidth]{fig/visualization/noisybest.pdf}}

\caption{Visualizations on Caltech101-20 (first row) and Noisy MNIST (second row) datasets with different methods. }
\label{visul}
\end{figure*}
\subsection{Proof of Theorem 1}
\begin{theorem}\label{the2}
Coding-consensus learning is equivalent to contrastive learning with positive pairs.
\end{theorem}
\begin{proof}
Considering the form of contrastive learning in~\cite{Hanijcai2024} as follows:
\begin{equation}
\begin{aligned}
\mathcal{L}_{\mathrm{cl}}=& -\frac{1}{|\mathcal{P}|}\sum_{(i,i)\in\mathcal{P}}\log s(\bm{z}_{i},\bm{z}_{i})  \\
&-\frac{1}{|\mathcal{N}|}\sum_{(i,j)\in\mathcal{N}}\log(1-s(\bm{z}_{i},\bm{z}_{j})), 
\end{aligned}
\tag{B2}
\end{equation}
where $\mathcal{P}$ and $\mathcal{N}$ represent the sets of positive and negative pairs, respectively, and $s$ denotes the similarity score for paired features as follows:
\begin{equation}\label{similarity}
s(\bm{z}_{i},\bm{z}_{i})=\exp\left(-\|\bm{z}_{i}-\bm{z}_{i}\|_{2}^{2}\right). \tag{B3}
\end{equation}

From a different perspective, the error between two-view features is governed by a Gaussian distribution due to the central limit theorem~\cite{BishopPRML}. Inspired by the work of~\cite{grill2020bootstrap}, the form of contrastive learning with positive is as follows,
\begin{equation}\label{cons2cl}
\mathcal{L}_{\mathrm{cl}}=\sum_{i}\|\bm{z}_{i}-\bm{z}_{i}\|_{2}^{2}, \tag{B4}
\end{equation}
Eq.~\eqref{cons2cl} measures the coding consensus between two different views, indicating that coding consensus learning is equivalent to contrastive learning with positive pairs.
\end{proof}
\subsection{Computational Complexity}
Let $b$, $h$, $l$, $n$ and $k$ denote the mini-batch size, the maximum number of hidden layers, the layer number, the number of samples, and the feature dimension, respectively. The complexity of forward propagation is:
\begin{equation*}
     \mathcal{O}(2nn_vd_vh^{(l+1)} ).
\end{equation*}
The complexities of reconstruction loss, classifying consensus loss, coding consensus loss and global consensus loss are as follows: 
\begin{equation*}
\begin{aligned}
     &\mathcal{O}(2nn_vd_v),\\
     &\mathcal{O}(n_v(n_v-1)(k^2b^2+3k^2)/2),\\
     &\mathcal{O}(n_vbk),\\
     &\mathcal{O}(n_v(n_v-1)bk).
\end{aligned}
\end{equation*}
The overall complexity is:
$\mathcal{O}(b_nn_v^2k^2b^2+2nn_vd_vh^{(l+1)})$, 
where $b_n=n/b$ is the maximum number of iterations. We obtain the final complexity:
\begin{equation*}
    \mathcal{O}(nn_v^2k^2b+2nn_vd_vh^{(l+1)}).
\end{equation*}

\begin{table*}[!t]
\centering
\caption{Hyperparameter of the two views clustering.}
\begin{tabular*}{0.88\textwidth}{@{\extracolsep{\fill}}l|lllllrl}
\toprule
Datasets  &$\alpha$& $\beta$ & $\gamma$ & $\lambda_1$   & $\lambda_2$ & $D_{\rm out}$ & $\rho$ \\ 
\midrule
Caltech101-20     &    3     &    3       &   8  &   0.1   &    0.1   &   128   &  0.10   \\ 
Scene-15 &     3.8    &      2.7     &   2.2  &   0.01   &  1     &    128  & 0.08  \\ 
LandUse-21  &    3     &       3.6    &  9.5   &   0.01   &    5   &  64    &  0.08 \\ 
Noisy MNIST    &    3     &      3     &   8  &   0.3   &   0.01    &  64    &  0.10  \\ 
\bottomrule
\end{tabular*}
\label{hyperpara}
\end{table*}

\begin{figure*}[!t]
\centering
\includegraphics[width=0.80\textwidth]{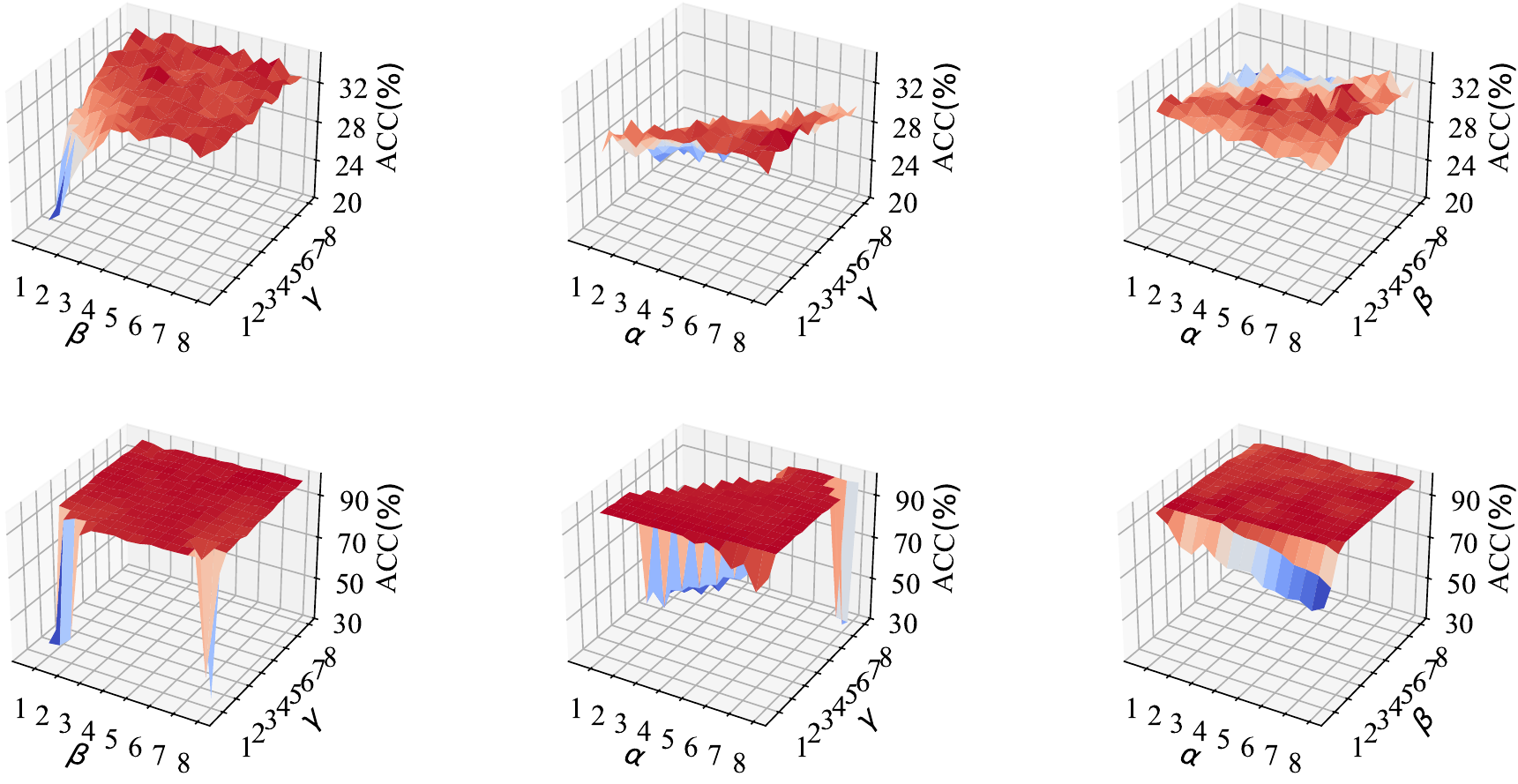}
\caption{Parameter sensitivity of $\alpha$, $\beta$ and $\gamma$ on LandUse-21 (first row) and Noisy MNIST (second row) datasets.}
\label{clspara}
\end{figure*}
\section{Experimental details}
\subsection{Hyperparameters Settings}
We implement the encoder with a four-layer MLP followed by a softmax layer to obtain the class probability. The decoders have the same architecture as the encoders. Specifically, the dimensions of each layer on the encoders are [$d_v$, 1024, 1024, 1024, $D_{\rm out}$], where $d_v$ is the input data dimension of the $v$-th view and $D_{\rm out}$ represents the latent representation dimension for different datasets.
In this section, we list the hyperparameters utilized in our multiview clustering model for each dataset as shown in Table~\ref{hyperpara}.
\subsection{Ablation Study}
To validate the importance of each component in HCN, we conducted ablation studies on the four datasets by discarding each component. The results are shown in Table~\ref{Ablation}, where DA means we apply classifying consensus objective without data augmentation. The ablation study results demonstrate all components of our HCN improve clustering performance.

\subsection{Visualization Analysis}
In this part, we provide the $t$-SNE visualizations of our method and seven baselines. Without loss of generality, we present the visualizations of Caltech101-20 and Noisy MNIST. As shown in Fig.~\ref{visul}, our method reveals more discriminative features among different classes than other baselines.
\subsection{Analysis of Hyperparameters $\alpha$, $\beta$ and $\gamma$}
In this part, we provide an analysis of parameters $\alpha$, $\beta$, and $\gamma$ with two views setting on LandUse-21 and Noisy MNIST datasets in Fig~\ref{clspara}. We fix a parameter as same as multiview clustering experiments, and the other two vary in a range of [1,8] with 0.5 intervals. The results show that our model is relatively stable with most combinations of trade-off hyperparameters of our classifying consensus learning objective.
\end{document}